\documentclass{article}

\usepackage[preprint]{neurips_2025}

\usepackage[utf8]{inputenc} 
\usepackage[T1]{fontenc}    
\usepackage{hyperref}       
\usepackage{url}            
\usepackage{booktabs}       
\usepackage{amsfonts}       
\usepackage{nicefrac}       
\usepackage{microtype}      
\usepackage{xcolor}         

\usepackage{amsthm} 
\usepackage{amsmath} 
\usepackage{mathtools}
\usepackage{amssymb}
\usepackage{float} 
\usepackage{caption} 

\usepackage{tabularx}
\usepackage{array}   

\newtheorem{definition}{Definition}
\newtheorem{theorem}{Theorem}

\newcommand{\TRUE}{\mathrm{TRUE}}
\newcommand{\FALSE}{\mathrm{FALSE}}  

\title{Polyra Swarms: A Shape-Based Approach to Machine Learning}

\author{%
  Simon Klüttermann\\
  Department of Computer Science\\
  TU Dortmund Univerity\\
  Dortmund, Germany \\
  \texttt{simon.kluettermann@cs.tu-dortmund.de} \\
  \And
  Emmanuel Müller \\
  Department of Computer Science \\
  TU Dortmund Univerity \\
  Dortmund, Germany\\
  \texttt{emmanuel.mueller@cs.tu-dortmund.de} \\
}

\begin{document}

\maketitle

\begin{abstract}
We propose Polyra Swarms, a novel machine-learning approach that approximates shapes instead of functions. Our method enables general-purpose learning with very low bias. In particular, we show that depending on the task, Polyra Swarms can be preferable compared to neural networks, especially for tasks like anomaly detection. We further introduce an automated abstraction mechanism that simplifies the complexity of a Polyra Swarm significantly, enhancing both their generalization and transparency. Since Polyra Swarms operate on fundamentally different principles than neural networks, they open up new research directions with distinct strengths and limitations.
\end{abstract}

\section{Introduction}

\begin{figure}[htbp]
    \centering
    \includegraphics[width=1.0\linewidth]{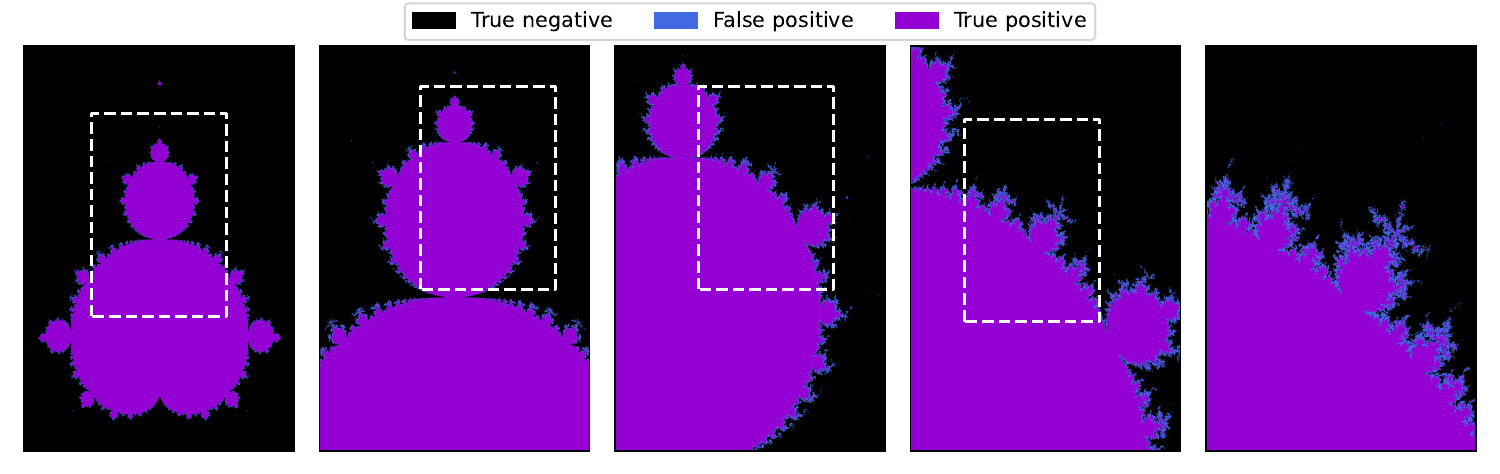}
    \caption{Example Polyra Swarm fit to a dataset inspired by the mandelbrot set (purple). To see the area falsely considered normal (blue), we need to zoom in further from left to right, showing the very low bias of our Polyra Swarm fits. Please note that false negatives are impossible here by construction. Further details about this experiment can be found in Appendix~\ref{app:mandelbrot}.}
    \label{fig:eye}
\end{figure}

Traditional machine learning frames learning as function approximation: inputs are mapped to outputs, and model success is evaluated by predictive accuracy. However, many human concepts, such as "dog" or "chair", are better characterized as regions within a structured feature space rather than as deterministic functions.

In this work, we propose a new learning paradigm: shape approximation. Instead of asking "What is the output corresponding to this input?", we ask "What region does this input belong to?" Our approach models such regions using combinations of semi-random polytopes connected via logical operations, forming what we call a Polyra Swarm.

Polyra Swarms construct regions by combining convex polytopes (via semi-random initialization) with logical operations (e.g., unions, intersections) to model non-convex structures, avoiding gradient-based optimization entirely and thus can not become stuck in local optima.
We demonstrate that they can approximate complex, high-dimensional structures with low bias and high flexibility (see Figure~\ref{fig:eye}). This geometric perspective enables general-purpose learning across tasks like classification, regression and anomaly detection.

We show that shape approximation has conceptual and or empirical benefits compared to traditional function approximation in several settings, particularly where structure, sparsity, or interpretability are important. Furthermore, the logical underpinnings of Polyra Swarms allow us to introduce symbolic abstraction techniques that simplify models into minimal logical rules, inherently regularizing models while improving both their interpretability and generalization.

While Polyra Swarms excel at precise, interpretable region modeling in low-to-moderate dimensions, they complement, rather than compete with—neural networks. By integrating representation learning (e.g., pre-trained embeddings) with shape approximation, we mitigate high-dimensional challenges while retaining geometric fidelity.

Learning through shape approximation offers a promising, logic-grounded alternative to function-based models, with potential implications, especially for shape-like tasks like anomaly detection but also for on-device computing, conceptual generalization, and interpretable machine learning.

Because of page limitations, we only briefly describe each experiment in the main paper and include further information and experiments in the appendix. Additionally, the code to reproduce all our experiments can be found at \href{https://anonymous.4open.science/r/polyra_exp-2BA7/README.md}{anonymous.4open.science/r/polyra\_exp-2BA7}. We also implement polyra swarms as a python library, which can be found at \href{https://anonymous.4open.science/r/polyra_lib-D1C0/README.md}{anonymous.4open.science/r/polyra\_lib-D1C0}. 
Unless explicitly mentioned, every experiment conducted here takes either seconds or minutes on a consumer-grade laptop.

\section{Related Work}\label{sec:rw}

Most machine learning methods to date have framed learning as a problem of function approximation, mapping inputs to outputs via models optimized for predictive accuracy~\cite{FundationPerceptron, FundationML}. This view aligns with the psychological doctrine of \emph{functionalism}, which emphasizes understanding the interaction of mental processes over their internal structure~\cite{FunctionalismIsCool, FunctionalismReview}. Functionalism emerged~\cite{DeweyFunctionalism} in contrast to \emph{structuralism}~\cite{StructuralismIntro}, which sought to analyze consciousness by identifying its constituent components.

Inspired by this philosophical distinction, we propose a structuralist perspective on machine learning. Rather than approximating a function, we aim to model the \emph{shape} of the data distribution itself. Specifically, we seek to identify the geometric regions in input space that are possible under a given distribution. We term this approach \emph{shape approximation}.

Several existing approaches partially reflect a structuralist inclination. One-class classification methods~\cite{surveyOneClass, ruff} learn a decision region occupied by a single class and can be seen as modeling structure directly in feature space. Feature engineering~\cite{FEfeatureengi} similarly involves decomposing behavior into meaningful low-level components. In symbolic learning~\cite{XSymbolic1,XSymbolic2}, prototype-based methods~\cite{SOMXPrototype}, and topological data analysis (TDA)~\cite{XTDA1,XTDA2}, researchers have explored structural descriptions of data. However, these methods typically require manual design or domain-specific tuning and have not achieved the level of automatic, general-purpose learning seen in function-oriented models such as deep neural networks.

Our approach bridges this gap by enabling general-purpose structural modeling: Polyra Swarms construct interpretable, logic-based shapes that can approximate complex data distributions without relying on gradient descent. This offers a new perspective that complements existing function-based paradigms.






\section{Polyra Swarms}
We will now describe Polyra Swarms mathematically and prove that they can approximate arbitrary shapes. Since our notation is extensive, we provide an overview in Appendix~\ref{app:notation}.

\begin{definition}[Shape]
    We consider a shape $Q$ to be a measurable, bounded subset of space $Q\subseteq \mathbb{R}^{dim}$ with a finite length border. A shape can be represented through an indicator function $q$.
    \begin{equation}
        q(x)=\begin{cases}
                \TRUE,\hspace{3em} x\in Q\\
                \FALSE,\hspace{2.6em} x\notin Q
            \end{cases}
            \label{eqn:indicator}
    \end{equation}
    \label{def:shape}
\end{definition}
We use capital letters to denote shapes, and lowercase letters to refer to indicator functions.

\begin{definition}[Shape Approximation]
    Given a set of samples $X_{train}$, which all lie in a shape $Q$ ($x\in Q,\;\forall x\in X_{train}$), the task of shape approximation is to learn an indicator function $f(x)$ so that the approximation error $\eta=\int_{\mathbb{R}^{dim}} \|f(x)-q(x)\| \,dx$ is minimal.
    \label{def:shapeApprox}
\end{definition}

\begin{definition}[Polytope]
    A common way of describing convex shapes is through polytopes. These are defined through $K$ constraint vectors $M_A \in \mathbb{R}^{K \times dim}$ defining directions and connected $b_A \in \mathbb{R}^K$ bounds that represent the maximum value along this direction inside the shape. Thus a polytope $A$ can be described by Equation~\ref{eqn:polytope}.
    \begin{equation}
        A = \bigcap_{j=1}^{K} \left\{ x \in \mathbb{R}^n : (M_A)^j \cdot x \leq (b_A)^j \right\}
        \label{eqn:polytope}
    \end{equation}
\end{definition}

\begin{definition}[Polyra Swarm]
    A \emph{Polyra Swarm}\footnote{Polyra is a portmanteau of "Polytope" and "Piranha". See Appendix~\ref{app:name}.} $P$ is the conjunction of many base shapes, as shown in Equation~\ref{eqn:merge}.
    \begin{equation}
        p(x)=\bigwedge_{i=0}^N f_i(x),\hspace{5em} P=\bigcap_{i=0}^N F_i
        \label{eqn:merge}
    \end{equation}
    A sample $x$ is in the shape approximated by a given Polyra Swarm, when it lies in each of the base shapes $F_i$
    \label{def:polyra}
\end{definition}

\begin{definition}[Polyra base shapes]
    We use base shapes $F_i$ defined by the logical conjunction of two polytopes (Equation~\ref{eqn:submodel}). If a sample lies in the condition polytope $A_i$, it also has to lie in the consequent polytope $B_i$.
    \begin{equation}
        f_i(x)=(x \in A_i \Rightarrow x \in B_i) \Leftrightarrow f_i(x)= (x \notin A_i \vee x\in B_i) \Leftrightarrow F_i=A_i^{\complement} \cup B_i
        \label{eqn:submodel}
    \end{equation}
    Here $A^{\complement}$ represents the complement shape to A ($a^{\complement}(x)=1-a(x)$).
    Effectively, each base shape excludes the shape $A_i \cap B_i^{\complement}$ from the shape approximated by the Polyra Swarm. 
    \label{def:submodel}
\end{definition}

\begin{theorem}
    For every measurable shape $Q$ and every $\epsilon > 0$, there exists a Polyra Swarm $P$ such that $\int_{\mathbb{R}^d} |p(x) - q(x)| dx < \epsilon$.
\end{theorem}
\begin{proof}
    $(Sketch)$ We can select arbitrarily small dim-simplices using the condition polytopes $A_i$ and remove the volume outside of $Q$ volume using the consequent polytopes $B_i$. $Full\;proof:$ Appendix~\ref{app:usaproof}
\end{proof}



We now propose an algorithm based on largely random initialization to learn such a swarm of conditional polytopes to approximate the shape containing a set of observed points $X_{train}$ in Equation~\ref{eqn:rnd}. Here $\mathcal{N}$ refers to a normal distribution and $\mathcal{U}$ represents a uniform distribution. Additionally, we leave $(K_A)_i$, $(K_B)_i$ and $N$ as hyperparameters to be discussed in Appendices~\ref{app:h:adim},~\ref{app:h:bdim},~\ref{app:hyper:ensemblesize}.

\begin{equation}
    (M_A)_i^{\mu\nu}, (M_B)_i^{\mu\nu} \overset{\text{i.i.d.}}{\sim} \mathcal{N}(0, 1) \hspace{3em}  (b_A)_i^{\mu} \overset{\text{i.i.d.}}{\sim} \mathcal{U}(\min_{x\in X_{train}}(M_A)_i^\mu\cdot x, \max_{x\in X_{train}}(M_A)_i^\mu\cdot x)
    \label{eqn:rnd}
\end{equation}

We reject those base shapes where $A_i$ does not contain any of the training samples (Appendix~\ref{app:h:minpoi}), and select $(b_B)_i$ from the most extreme values allowed in the condition polytope $A_i$.
\begin{equation}
    (b_B)_i^\mu=\max_{(x\in X_{train}) \;\cup\; ((M_A)_i\cdot x\leq (b_A)_i)} (M_B)_i^\mu\cdot x
    \label{eqn:maximaB}
\end{equation}
For symmetry reasons, we also include the second consequent region induced by $(M_A)_i\rightarrow-(M_A)_i$.

This setup can be learned quickly, is easily parallelized, and guarantees that $p(x)=True\;\forall x\in X_{train}$ ($Proof:$ Appendix~\ref{app:alltrueproof}).

\section{Using Polyra Swarms}\label{sec:use}
After showing that it is possible to approximate any shape using a Polyra Swarm, here we want to demonstrate that universal shape approximation also implies general-purpose learning. Doing so generally requires rethinking how to approach any machine-learning problem. In this section, we will discuss some of the most common machine-learning tasks and show how they can be solved using Polyra Swarms.

\subsection{Shape Approximation}\label{sec:match}
\begin{figure}[htbp]
    \centering
    \includegraphics[width=0.6\linewidth]{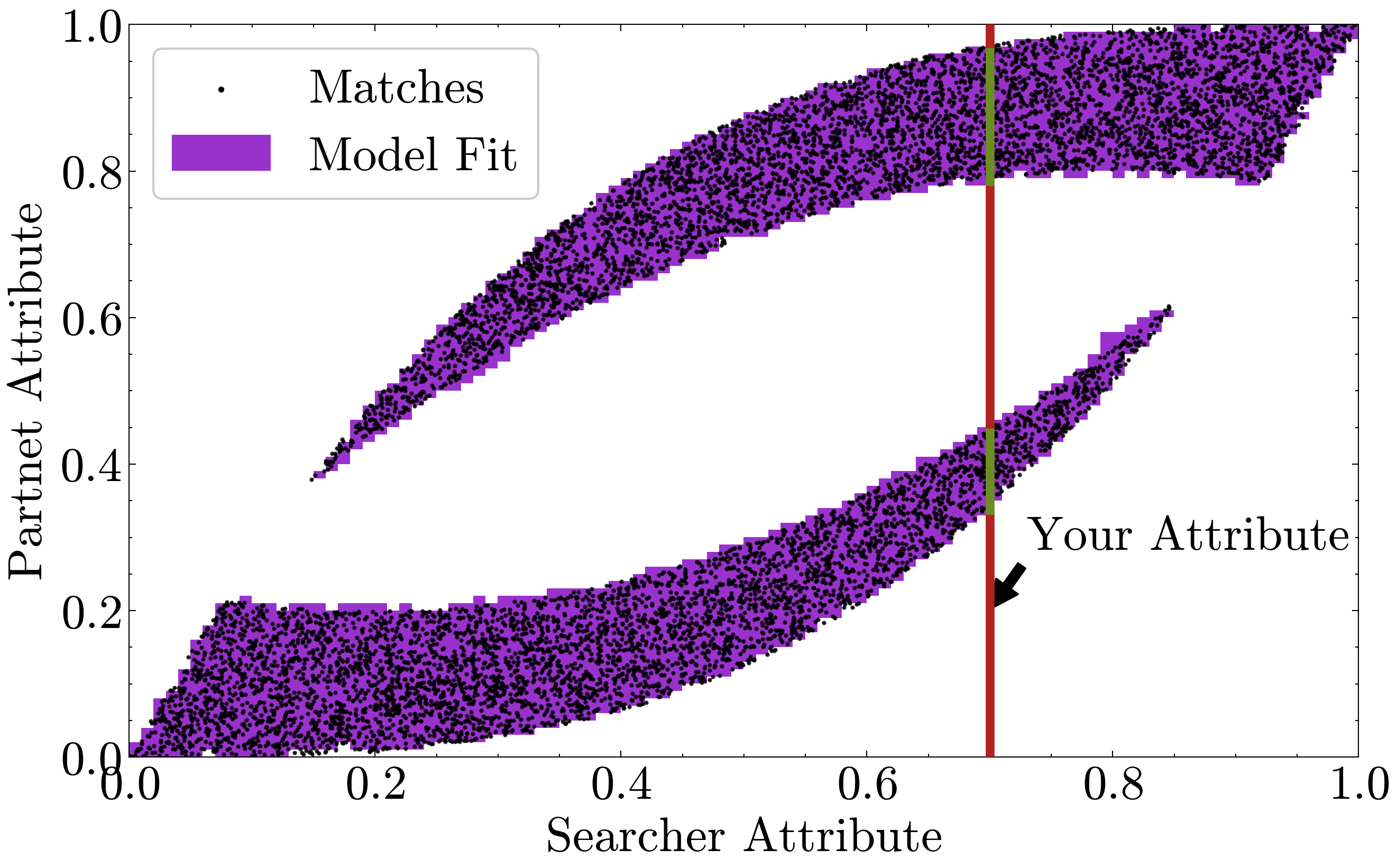}
    \caption{Example of how to solve matching using shape approximation. Given a pattern between your and your partner's attributes, we can describe it as a shape. More details about the experimental setup are found in Appendix~\ref{app:match}.}
    \label{fig:match}
\end{figure}

First, consider that many machine learning problems are more naturally shape approximation tasks. As an example, consider the toy example in Figure~\ref{fig:match} showing how to solve matching problems using shape approximation.
The task of selecting a well-matching partner depending on your and their attributes is a common problem from advertising to job matching. It can be solved through, e.g., siamese neural networks~\cite{siamese}, learning a representation in which matching samples are close to each other. However, instead of learning such a representation, we search for the shape in the space of both your and your match attributes of participants who match successfully.


\subsection{Classification}\label{sec:class}
\begin{figure}[htbp]
    \centering
    \begin{minipage}[b]{0.32\linewidth}
        \centering
        \includegraphics[width=\linewidth]{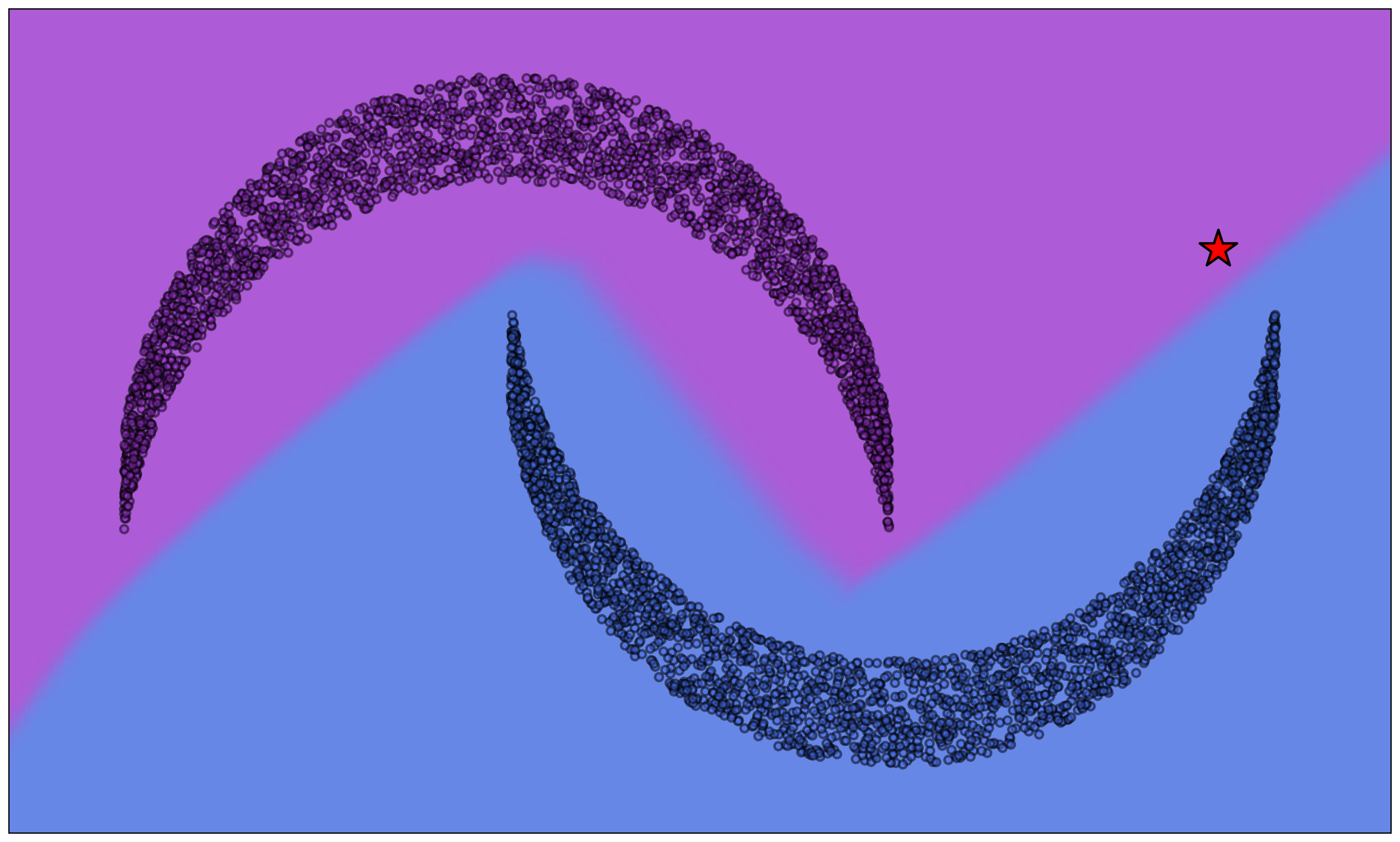}
        \caption*{(a) Using a neural network}
    \end{minipage}
    \hfill
    \begin{minipage}[b]{0.32\linewidth}
        \centering
        \includegraphics[width=\linewidth]{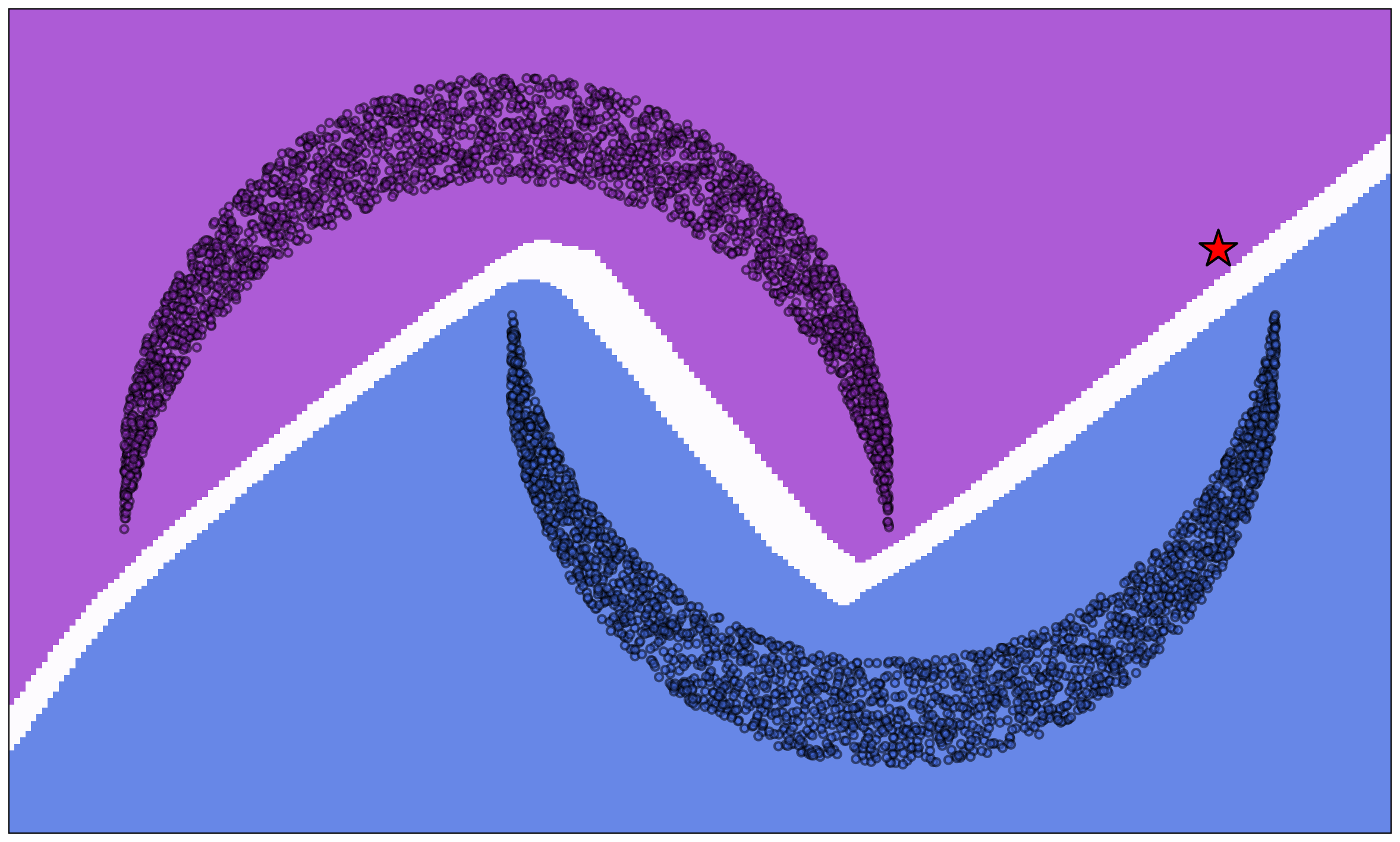}
        \caption*{(b) and a reject option}
    \end{minipage}
    \hfill
    \begin{minipage}[b]{0.32\linewidth}
        \centering
        \includegraphics[width=\linewidth]{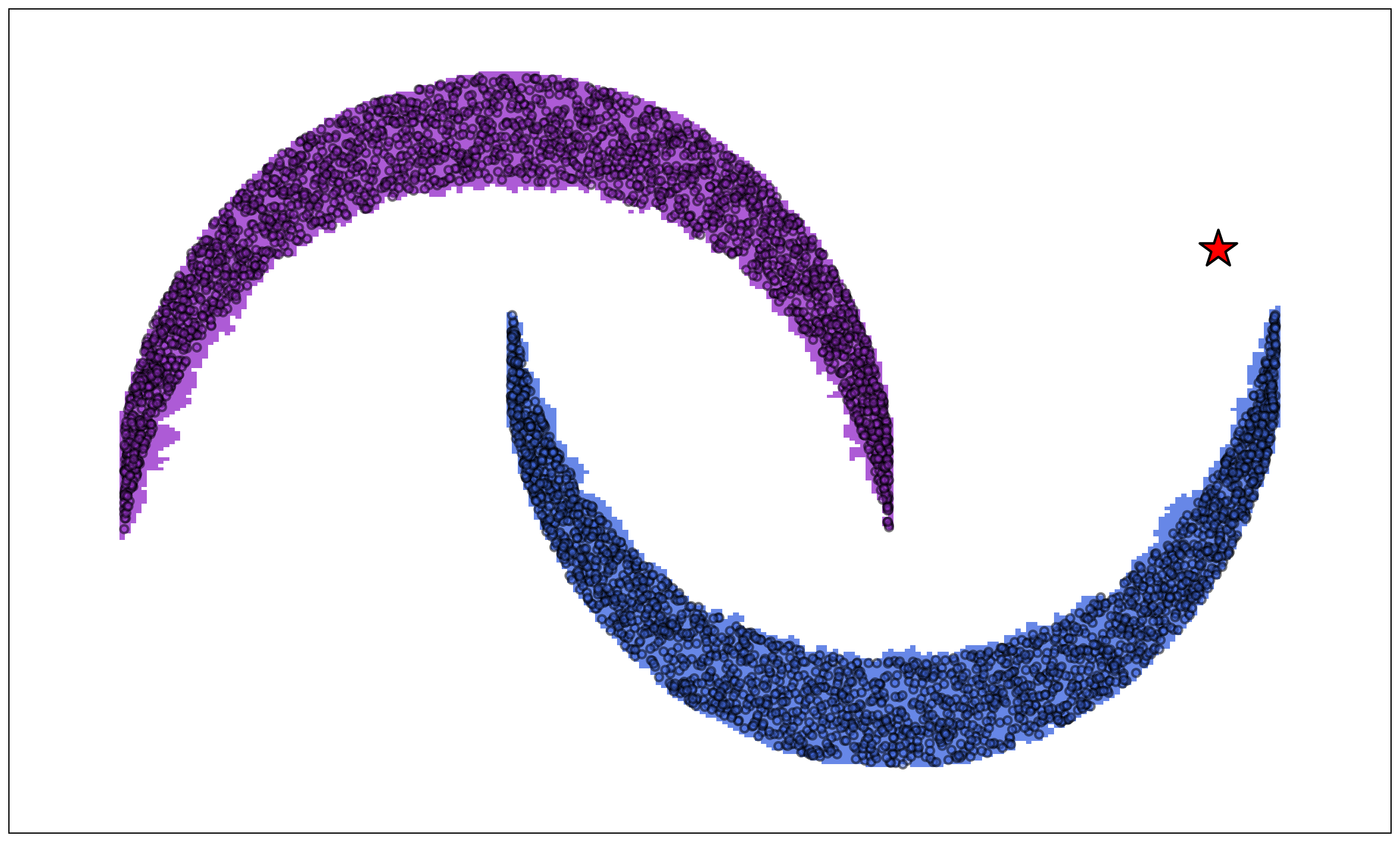}
        \caption*{(c) Using two Polyra Swarms.}
    \end{minipage}
    \caption{Classification results on the two moons dataset using various models. The neural network on the left generally classifies every point in input space, often without basis. For example, the red star is much closer to blue points than to purple ones but is still classified as purple. The Polyra swarms on the right side instead only identify the region that matches one class and thus reject samples outside of these regions. More details of our experimental setup can be found in Appendix~\ref{app:classification}. }
    \label{fig:class}
\end{figure}
A shape approximation model generally answers the question: "Does this sample belong to a given class". This is very close to the goal of classification algorithms, and thus we can use one Polyra swarm to identify each class. 
This has some benefits and drawbacks over a function approximation algorithm approximating the decision boundary. First of all, since we are not learning the boundary between two regions, but describe different classes each, we can add and remove classes without needing to retrain the whole model. A classification setup trained to separate dogs from cats can be partially reused to separate dogs from horses. Additionally, as Figure~\ref{fig:class} shows, this inherently includes a very effective reject option~\cite{rejectsurvey}. However, this might not always be desirable. In contrast to the neural network model, which classifies any sample as either class 1 or 0, a Polyra swarm classifier can also classify a sample as "neither" (reject option) or as "could be both". While the last case does not happen in Figure~\ref{fig:class}, it is prevalent in real applications like the MNIST classification we consider in Appendix~\ref{app:mnistclass}. 



\subsection{Anomaly Detection}\label{sec:ad}
\begin{figure}[htbp]
    \centering
    \includegraphics[width=0.8\linewidth]{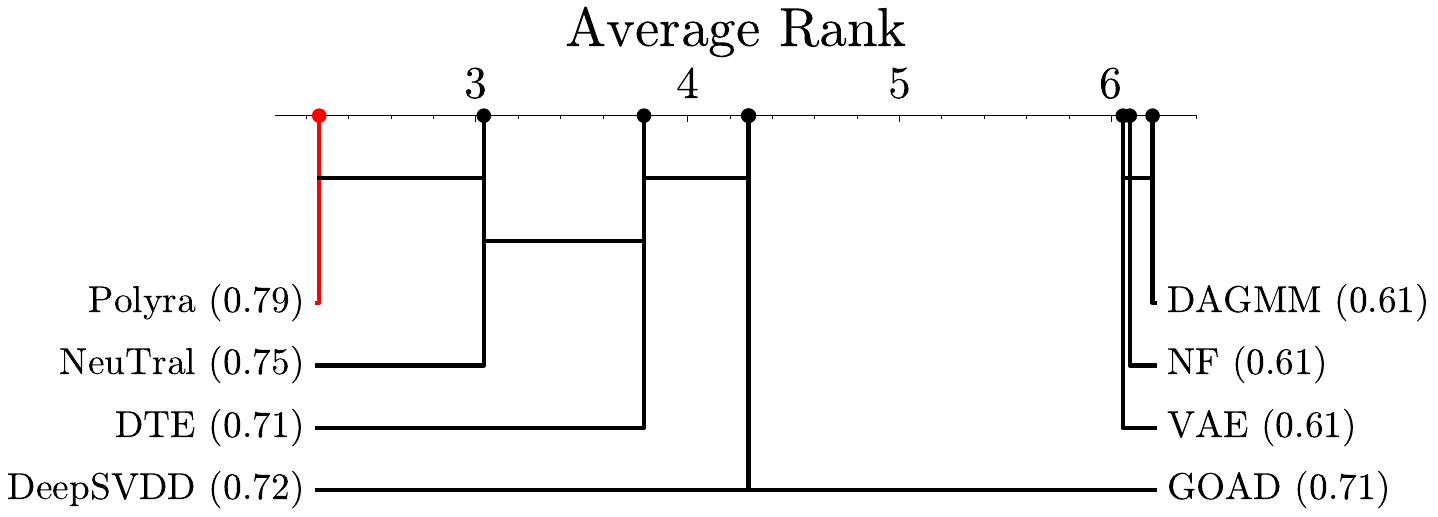}
    \caption{Anomaly detection critical difference plot comparing Polyra swarms to recent/common deep learning anomaly detection methods (Using a Friedmann~\cite{friedman} and Wilcoxon~\cite{wilcoxon} test and a Bonferroni-Holm correction~\cite{correction}). We use 121 common benchmark datasets taken from a recent survey~\cite{surveyzhao} in the semi-supervised setting. As anomaly detection is usually evaluated with continuous anomaly scores (See Appendix~\ref{app:rocbias}), we calculate the fraction of submodels that consider a sample invalid as anomaly score. More about this experiment can be found in Appendix~\ref{app:anomaly}. }
    \label{fig:better}
\end{figure}

Anomaly Detection in general, tries to answer whether a sample belongs to the class "normal" or not. Thus shape approximation is perfectly suited to anomaly detection. We evaluate this in Figure~\ref{fig:better} using 7 recent neural network (and thus function approximation-based) anomaly detection algorithms on the benchmark datasets proposed by a recent survey paper~\cite{surveyzhao}. We find that Polyra Swarms outperforms all our competitors and does so significantly for most of them.

\subsection{Regression and Uncertainty Estimation}\label{sec:reg}
\begin{figure}[htbp]
    \centering
    \includegraphics[width=0.8\linewidth]{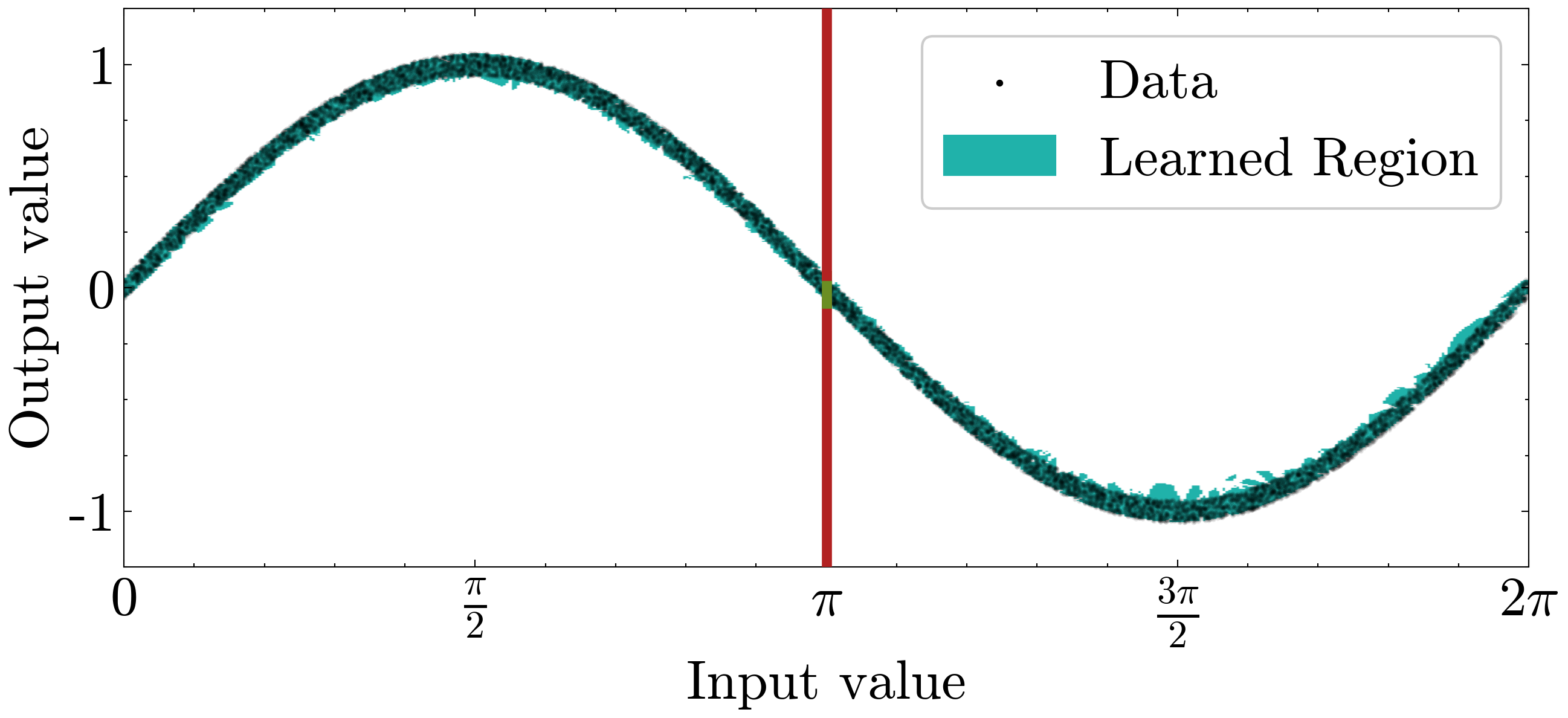}
    \caption{Example of a function ($\sin(x)+5\%\textrm{ error}$) approximated by a Polyra swarm. For $x=\pi$, the valid range of the model (green region) lies between $-0.070$ and $0.054$, closely following the expectation of $0\pm0.05$. Please note that the generated range is exact, but requires the Rangefinder algorithm as explained in Appendix~\ref{app:abs:1d} and is thus relatively slow ($O(0.1s)$). More information about this experiment can be found in Appendix~\ref{app:regression}.}
    \label{fig:reg}
\end{figure}

While tasks like classification and anomaly detection are well suited to shape approximation, regression, as the task of approximating a function is more suited to function approximation. However, every function \( f : \mathbb{R}^d \to \mathbb{R}^n \)
is still also a shape in $\mathbb{R}^{d+n}$, which we can approximate. Interestingly, while doing so we also approximate the uncertainty of the curve (See Figure~\ref{fig:reg}).

\subsection{Point generation}\label{sec:gen}

\begin{figure}[H]
    \centering
    \includegraphics[width=0.4\linewidth]{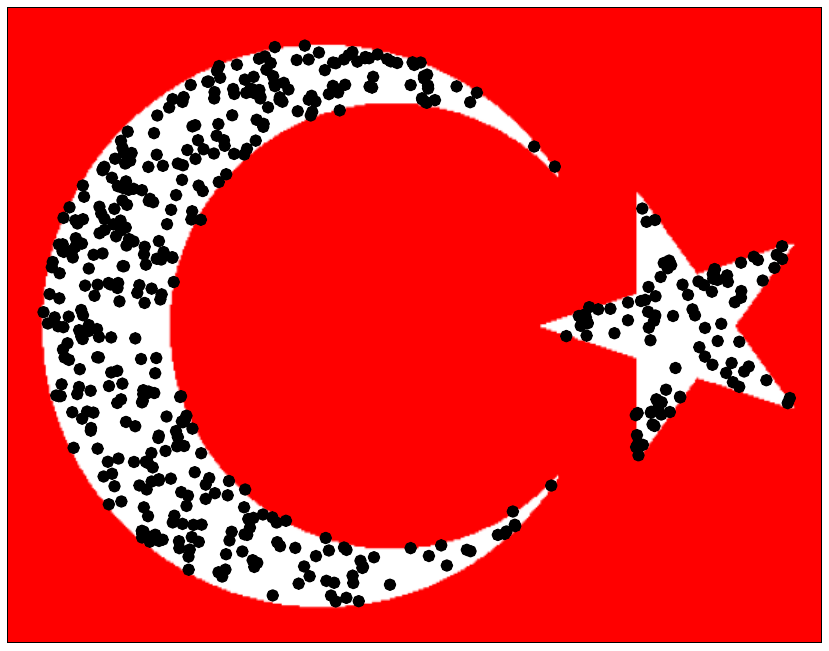}
    \caption{Example of samples generated from a Polyra swarm. We first fit a swarm to represent the white part of the flag of Türkiye and then generate random samples from this model. These samples fit well to the original flag and are visually uniformly distributed over both the half-moon and the star. More information about this experiment can be found in Appendix~\ref{app:generative}. Experiments on higher dimensional samples ($>2$) can be found in Appendix~\ref{app:zerogen}.}
    \label{fig:turk}
\end{figure}

Most generative models work by considering random noise and learning to transform it into samples following the desired distribution~\cite{diffusionCool}. As we can not learn such transformations, this is impossible with Polyra swarms. However, we can more directly test whether a sample fulfills $p(x)=\TRUE$. Thus generative polyra swarms are algorithms that efficiently generate potential candidates inside a given shape. Our example in Figure~\ref{fig:turk} uses a modification of the HitAndRun algorithm~\cite{hitandrun} to work with non-convex shapes (See Appendix~\ref{app:generative}). 

\section{Abstraction}

One of the biggest benefits of neural networks and, thus, function approximation is the ability to automatically learn higher-level features during training~\cite{AbstractionAcrossLayers}. These tend to generalize better~\cite{measuringAbstractionGeneralisation} and are a core reason why neural networks work well on high-dimensional data.

While Polyra Swarms can not learn such features during training, they can still do something similar. A Polyra Swarm is effectively a tree of logical expressions and we can search for other logical trees that describe approximately the same shape as a given Polyra swarm. Our algorithm for this is rather complicated and described in Appendices~\ref{app:abs}, \ref{app:abs:1d}, \ref{app:abs:sample} and \ref{app:abs:lp}. In short, our algorithm converts the logical tree describing a Polyra swarm into disjunctive normal form while removing mainly redundant terms.

An example of what this can do, is shown in Figure~\ref{fig:defrag}. The left side shows an effect similar to overfitting in neural networks, where the shape contains many holes, which simply did not contain any training samples. We call this effect $fragmentation$. Starting from a quite fragmented shape, which is described by $2000$ submodels, the automatically abstracted version on the right side is described by multiple orders less parameters and fits the ground truth significantly better. In fact, this logical tree is simple enough for us to include it in this paper (See Equation~\ref{eqn:abs:example}), in contrast to the black-box problem of neural network approximations~\cite{blackbox}.

\begin{figure}[htbp]
    \centering
    \includegraphics[width=0.8\linewidth]{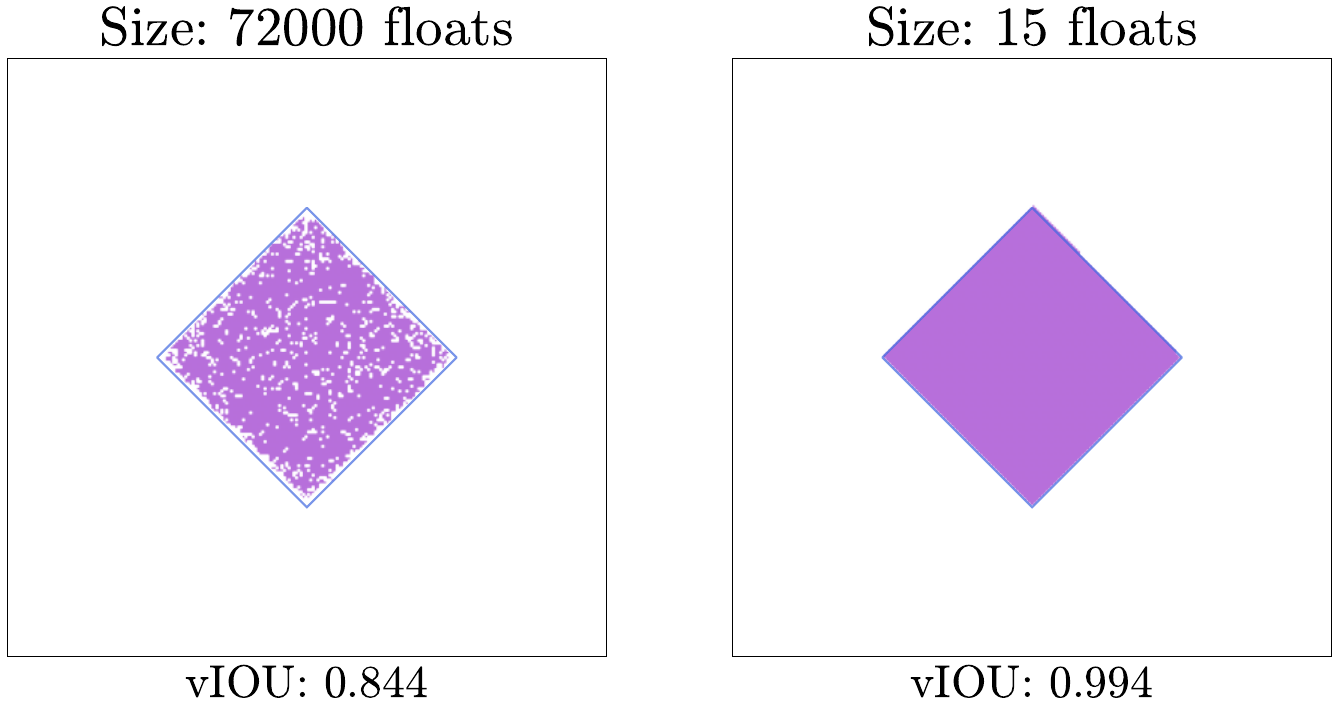}
    \caption{Example of the benefits abstraction can provide. We fit a Polyra swarm to samples uniformly distributed over a simple diamond shape. The learned shape (left plot) does not fit the shape perfectly, as shown by the relatively low volumetricIOU (see Appendix~\ref{app:vIOU}). This could be somewhat manually improved with better hyperparameters (see Appendix~\ref{app:ablation}). Instead, we use our abstraction algorithm (See Appendix~\ref{app:abs:sample}) to produce the right plot, which fits the ground truth close to perfectly. At the same time, we reduce the complexity of our model by a factor $4800$, resulting in a model that is both very accurate, fast, and simple enough to be printed into this paper (Equation~\ref{eqn:abstraction:example:diamond}). More information about this experiment can be found in Appendix~\ref{app:diamond}.}
    \label{fig:defrag}
\end{figure}

\begin{equation}
\begin{bmatrix}
-0.7071 & -0.7072\\ 
-0.7081 & 0.7061\\ 
0.7086 & -0.7056\\ 
0.7155 & 0.6986\\ 
0.7853 & 0.6192\end{bmatrix}\\ 
 \cdot x \leq 
\begin{bmatrix}
-0.5303\\ 
0.1759\\ 
0.1784\\ 
0.8852\\ 
0.8967\end{bmatrix}
\label{eqn:abstraction:example:diamond}
\end{equation}

While the diamond-shaped toy data in Figure~\ref{fig:defrag} is constructed to be easily described by a few halfspace conditions, the same abstraction process also works on real-world datasets. To show this, we consider the classical OldFaithful dataset~\cite{oldfaithful} in Figure~\ref{fig:rw} and the abstracted fit in Equation~\ref{eqn:abstraction:example:rw}.

\begin{figure}[htbp]
    \centering
    \includegraphics[width=0.8\linewidth]{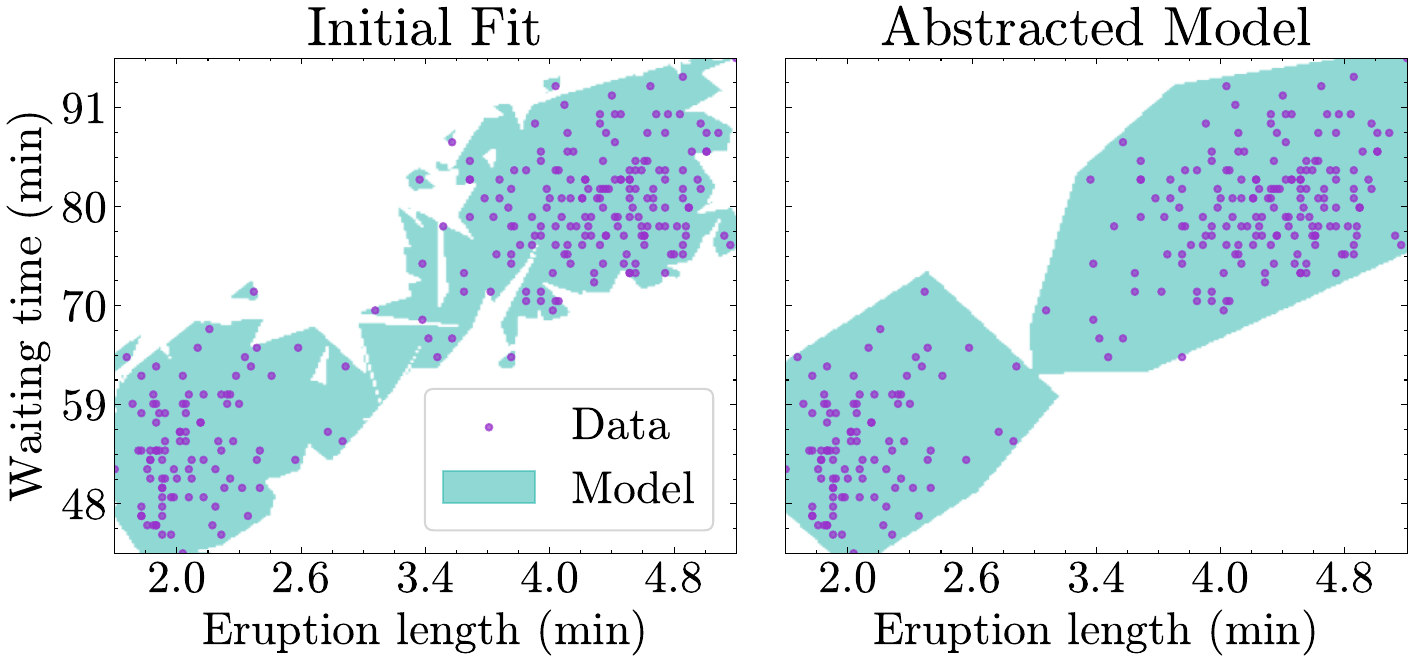}
    \caption{Equivalent of Figure~\ref{fig:defrag} on a real-world dataset. Left: Shape approximation of the OldFaithful dataset~\cite{oldfaithful}. Right: Abstracted version of the same model (See Equation~\ref{eqn:abstraction:example:rw}). More information about this experiment can be found in Appendix~\ref{app:faith}.}
    \label{fig:rw}
\end{figure}
\begin{equation}
\left( \begin{bmatrix}
-0.756 & -0.6546\\ 
0.6207 & -0.784\\ 
0.8259 & -0.5639\\ 
-0.6207 & 0.784\\ 
0.7643 & 0.6448\end{bmatrix}\\ 
 \cdot x \leq 
\begin{bmatrix}
-0.0522\\ 
0.095\\ 
0.1859\\ 
0.3054\\ 
0.5405\end{bmatrix}\right)\\ 
 \bigvee 
\left( \begin{bmatrix}
0.4978 & -0.8673\\ 
-0.7498 & 0.6617\\ 
-0.1934 & 0.9811\\ 
-0.9716 & 0.2366\\ 
0.0267 & -0.9996\\ 
-0.9982 & -0.0606\end{bmatrix}\\ 
 \cdot x \leq 
\begin{bmatrix}
-0.0259\\ 
0.1568\\ 
0.8063\\ 
-0.2722\\ 
-0.3493\\ 
-0.4175\end{bmatrix}\right)
\label{eqn:abstraction:example:rw}
\end{equation}

The abstracted shape in Equation~\ref{eqn:abstraction:example:rw} consists of the union of two polytopes representing the two clusters of points in the OldFaithful dataset. We explore further if we can use abstraction as a clustering algorithm in Appendix~\ref{app:clustering}.

\section{Conclusion, Limitations and Comparison to Neural Networks}\label{sec:conclusion}

\begin{figure}[htbp]
    \centering
    \begin{minipage}[b]{0.48\linewidth}
        \centering
        \includegraphics[width=\linewidth]{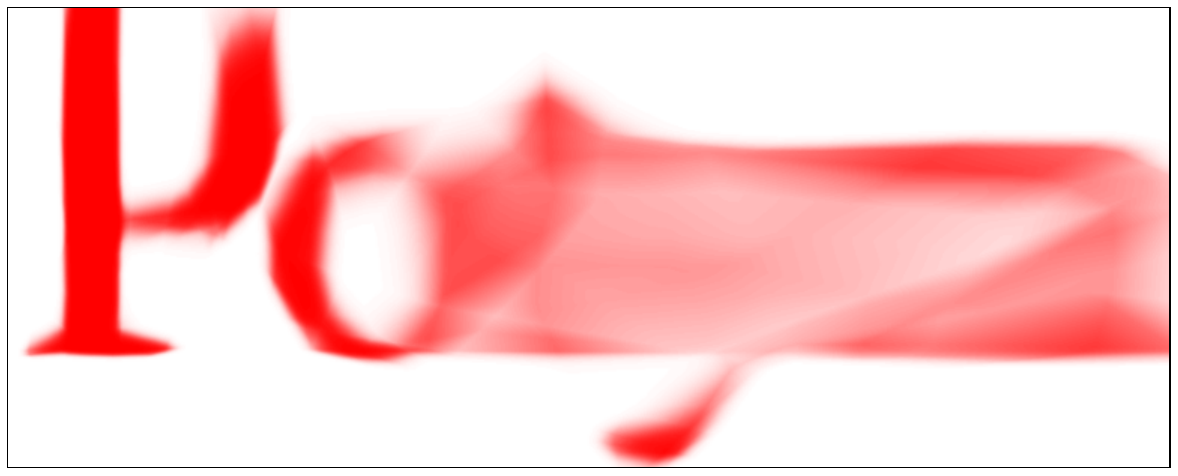}
        \caption*{(a) Using a neural network}
    \end{minipage}
    \hfill
    \begin{minipage}[b]{0.48\linewidth}
        \centering
        \includegraphics[width=\linewidth]{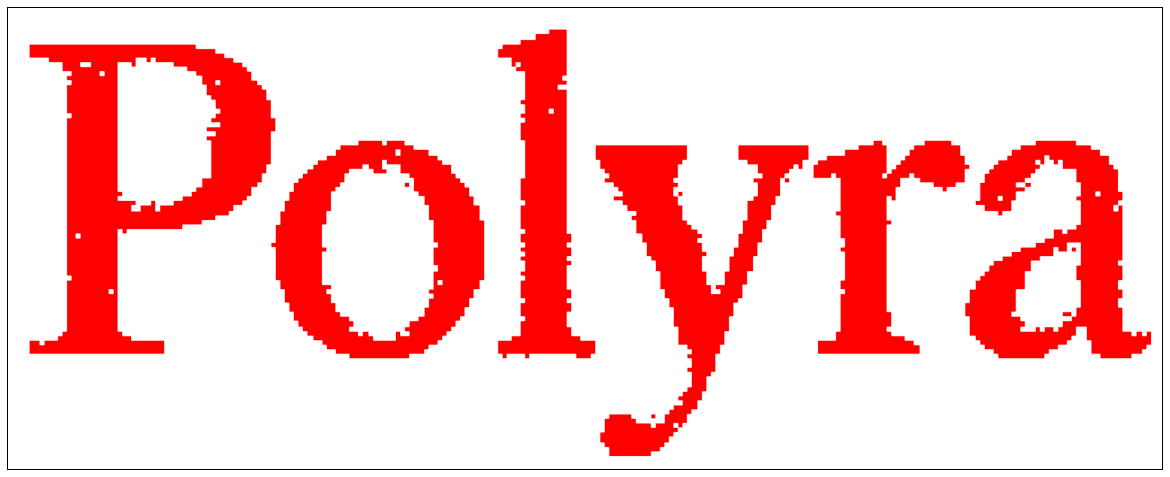}
        \caption*{(b) Using a Polyra Swarm.}
    \end{minipage}
    \caption{Comparison of the precision of a Polyra swarm to a neural network. We search for a function that can decide if a two-dimensional point lies in the graphical representation of the word "Polyra" or not. While both models theoretically are expressive enough to model such a highly complicated function, the neural network gets stuck in a local minima and only the polyra swarm learns a shape representing readable text. More information about this experiment can be found in Appendix~\ref{app:textapprox}.}
    \label{fig:text}
\end{figure}

We have shown that polyra swarms can be very useful in certain situations and have nice properties: Because we learn shapes instead of functions, there are tasks like anomaly detection and classification that polyra swarms are more suited to. Similarly, since we do not rely on gradient-based optimization, it is impossible to get stuck in local minima, resulting in an often more precise model (See Figure~\ref{fig:text}). Additionally, using abstraction, we can increase the generalization capabilities, the explainability, and the prediction speed of a learned model. However, polyra swarms are not without faults. To help mitigate them and provide a fair comparison to neural networks, we will mention them in this section.

\subsection{High Dimensional Data - Precision instead of Scale}\label{sec:c:highdim}
Most of the experiments in this paper have been conducted in two dimensions only. This is partially because this allows visualizing our results. However, it is also the case that polyra swarms do not handle high-dimensional shapes equally well. Following the proof in Appendix~\ref{app:usaproof}, we require at least $dim+1$ conditions per condition polytope to approximate arbitrary shapes. Additionally, generative models tend to struggle with high-dimensional data~\ref{app:zerogen}, and so far, the fastest general abstraction algorithm has a runtime that is at least cubic in $dim$ (Appendix~\ref{app:abs:lp}).
Thus neural networks are generally better for high-dimensional data.

However, we find it debatable if it ever should be the goal to replace the very large amount of research that has been put into neural networks already.
Instead, we believe there to be tasks where neural networks and function approximation are best suited, and tasks where shape approximation methods like Polyra Swarms are superior.
Furthermore, it is possible to combine the generally highly precise Polyra Swarms with existing feature representation algorithms, like PCA~\cite{pca}, autoencoder~\cite{aeusedimensionality} or contrastive approaches~\cite{personReidSurvey}. This reduces the dimensionality and cancels out the limitation of Polyra Swarms. We use this in Appendices~\ref{app:mnistclass} and \ref{app:zerogen}.


\subsection{Bias towards Convex shapes}\label{sec:c:convex}
Since we represent the base regions ($A_i$,$B_i$) in Equation~\ref{eqn:submodel} as polytopes, which represent convex regions, the resulting fits also have a bias towards learning locally convex shapes (In Figure~\ref{fig:text} there are more errors inside enclosed regions then next to them. This is even more visible in Appendix~\ref{app:opt}).
This is a bias since it unreasonably favors one type of shape over another.
However, since overfitting/fragmentation represents itself also as holes in a learned shape (Figure~\ref{fig:defrag}), it is often beneficial to focus on convex shapes. We include a hyperparameter to control the tradeoff between the convexity and complexity of a learned shape (See Appendix~\ref{app:h:minpoi}).

\subsection{Long-tail distributions}\label{sec:c:longtail}
We generally try to learn a region that contains all possible points. This makes distributions with long tails (like a normal distribution) conceptually challenging. Since the region in which a point might theoretically be observed is often arbitrarily large, we are not able to learn any meaningful shape here. In practice, the likelihood of points far away from the center is commonly exponentially suppressed, and thus, we are still able to learn a shape containing all observed points. The limitation lies in the fact that this shape will be very fragmented in low-density regions and will not generalize well to other samples observed by the same distribution. However, as Appendix~\ref{app:linearisation} shows, this can be mitigated by a clever choice of hyperparameters.




\section{Future Work}

To propose an algorithm achieving general-purpose learning requires considering many different fields. While we believe it to be interesting to focus on the way applying shape approximation requires rethinking existing chains of thought, this also means we could only experiment with most tasks qualitatively in the main paper. Thus, we believe this to only be the start of the research into polyra-like shape approximation algorithms. Many decades of research time have been spent trying to understand and improve neural networks. We believe it is proof of the large potential that Polyra Swarms have, that they can already be competitive with neural networks on certain tasks. To further help with this research, we want to suggest a few research directions.

We implement all functionality used in the experiments of this paper into a python library, allowing any researcher to use polyra swarms for their own tasks. And while the runtime of polyra swarms has not been a problem so far, we could likely still make this library more efficient by for example implementing GPU support.
Similarly, we are interested in different base regions. Polytopes generally become more complicated in higher dimensions, limiting their applicability without feature extraction. Instead, when considering, for example, hyperspheres, this dependency disappears. Additionally, similar to convolutional layers for neural networks there might be base-regions that are optimal for different types of data.
Finally, our abstraction algorithms are currently more heuristic than principled, with many design choices significantly influencing their performance. A more systematic study of these choices could lead to substantially more effective abstraction methods.






\bibliographystyle{plainnat}
\bibliography{thesis,posi,fp,new}








\appendix

\section{Notation Cheat Sheet}\label{app:notation}
To help understanding our work, we provide a summary of our used notation in the following table.

\noindent
\begin{tabularx}{\textwidth}{>{\raggedright\arraybackslash}X >{\raggedleft\arraybackslash}m{3cm}}
\textbf{Description} & \textbf{Symbol} \\
\hline
Shape & $Q$ (capital letter)\\
Indicator function & $q(x)$ (lowercase letter)\\
Set of training samples & $X_{train}$\\
Polyra Swarm & $P$,$p(x)$\\
Base-shape & $F_i$, $f_i(x)$\\
Condition polytope & $A_i$\\
Consequent polytope & $B_i$\\
Approximation error & $\eta$\\
Number of conditions in a Polytope & $K$\\
Polytope constraint matrix & $M$\\
Polytope bounds vector & $b$\\
Normal distribution & $\sim\mathcal{N}(\textrm{mean}, \textrm{std})$\\
Uniform distribution & $\sim\mathcal{U}(\textrm{min}, \textrm{max})$\\
Complement of a shape & $Q^{\complement}$\\
\label{tab:note}
\end{tabularx}

\section{Proof of Universal Shape Approximation through Polyra Swarms}\label{app:usaproof}

\begin{proof}
    
    From Definition~\ref{def:shape}, we know that a shape is a measurable bounded subset of space with a finite length boundary.
    To prove that a Polyra Swarm can approximate such a set $S$, we rewrite Equation~\ref{eqn:merge} into
    \begin{equation}
        P=\left(\bigcup_{i=0}^N F_i^\complement\right)^{\complement}=\left(\bigcup_{i=0}^N (A_i\setminus B_i)\right)^\complement
        \label{eqn:altmerge}
    \end{equation}
    Thus, each base shape effectively carves out a region that it considers impossible. We can use $2\cdot dim$ such submodels to reduce every point where one feature $h<\min_{h}(S)$ or $h>\max_{h}(S)$. Since $S$ is bounded, the resulting shape $P_{approx}$ is finite and we refer to its volume as $V$. Any such finite, square region can be entirely filled with dim-simplices of volume $\nu$. We choose the remaining $A_i$ to represent these simplices (Thus $(K_A)_i=dim+1$) and the matching $B_i$ so that $A_i \cap B_i$ approximates $A_i \cap S$ best. 
    
    Doing so, for each dim-simplex $A_i$ there are three options. 
    \begin{enumerate}
        \item If $A_i \subseteq S$, we can choose $B_i=A_i$ and incur zero error ($\eta=0$).
        \item Similarly, if $A_i \cap S = \emptyset$, we can choose $B_i=\emptyset$ and also incur zero error ($\eta=0$).
        \item Only if $A_i\cap S \neq \emptyset $ and $ \neq A_i \subseteq S$ we need to choose a possibly imperfect $B_i$ and might incur an error. Importantly this error is still bounded by the volume $\nu$.
    \end{enumerate}
    So the total error of our approximation is bounded by $\nu\cdot C(\nu)$, where $C(\nu)$ is the number of dim-simplices of volume $\nu$ that intersect the boundary.

    We want to show now that this term becomes arbitrarily small when we reduce $\nu$: $\lim_{\nu\rightarrow 0} \nu\cdot C(\nu)=0$.
    Since the boundary has finite length, it can be locally well approximated by a (d-1)-dimensional hyperplane for sufficiently small $\nu$.

    We assume each dim-simplex to have approximately equal side-lengths. This means, we expect in each dimension there to be $\left(\frac{V}{\nu}\right)^{\frac{1}{dim}}$ dim-simplices, of which every hyperplane only intersects $O(1)$ many and thus the fraction $C(\nu)\propto O(\nu^{-\frac{dim-1}{dim}})$.

    Combining this together, we find that 
    \begin{equation}
        \lim_{\nu\rightarrow 0} \nu\cdot C(\nu)\propto\lim_{\nu\rightarrow 0} \nu\cdot \nu^{-\frac{dim-1}{dim}} =0
    \end{equation}

    Thus the erroneous volume $\eta=\|S\setminus P \cup P\setminus S\|$ goes to zero as $\nu\rightarrow 0 \Leftrightarrow N\rightarrow\infty$, proving our conjecture.

\end{proof}

\section{Proof that all training samples are included in the shape learned by a Polyra Swarm}\label{app:alltrueproof}
\begin{proof}
    Assume that there would be a point $x\in X_{train}\text{, where } p(x)=\FALSE$.
    
    This means, there is at least one base shape $\exists i\text{, so that }f_i(x)=\FALSE$. 
    
    Further, this implies $x\in A_i$ and $x\notin B_i$.
    
    Focussing on the second part, this means that there $\exists j\text{, so that } (M_B)_i^j \cdot x > (b_B)_i^j$.

    However, following Equation~\ref{eqn:maximaB}, we know that $(b_B)_i^j=\max_{ (x\in X_{train}) \;\cup\; (x\in A_i)} (M_B)_i^j\cdot x$. And since $x\in X_{train}$ and $x\in A_i$, we know that $(b_B)_i^j\ge x$ and this is contradictory. Thus every sample in the training set $X_{train}$ lies in the shape $P$.
\end{proof}

\section{Volumetric IOU}\label{app:vIOU}
To evaluate whether a learned shape (L) fits a given ground truth shape (T), many metrics have been suggested to compare such sets~\cite{vIOU}. Here, we use the Volumetric Intersection Over Union. It is defined as
\begin{equation}
    \textrm{vIOU}=\frac{\textrm{Volume}(L\cap T)}{\textrm{Volume}(L\cup T)}
    \label{eqn:vIOU}
\end{equation}

In case of a perfect fit, $L=T$ and thus $\textrm{vIOU}=1$, while in the worst case, there is no overlap between $L$ and $T$, and thus $\textrm{vIOU}=0$.

\section{Influence of various Hyperparameters and Ablation Studies}\label{app:ablation}

Polyra Swarms generally contain many hyperparameters. In each of our experiments, we typically show only the result of a well-parametrized swarm. We keep from optimizing these hyperparameters excessively to show reasonable results possible using Polyra Swarms with limited effort. Still, we believe it is important to show the effect each hyperparameter has, which is why we study variations of each parameter in this section on a toy dataset.

This toy dataset is shown in Figure~\ref{fig:inihyper} together with a fit using default hyperparameters and another fit using a more reasonable choice of hyperparameters (minpoi: $500$, extend: $0.1$, submodel-count: $10000$). The more reasonable set of hyperparameters produces a better-looking fit, which is also confirmed using the vIOU metric (Appendix~\ref{app:vIOU}).

\begin{figure}[H]
    \centering
    \includegraphics[width=0.6\linewidth]{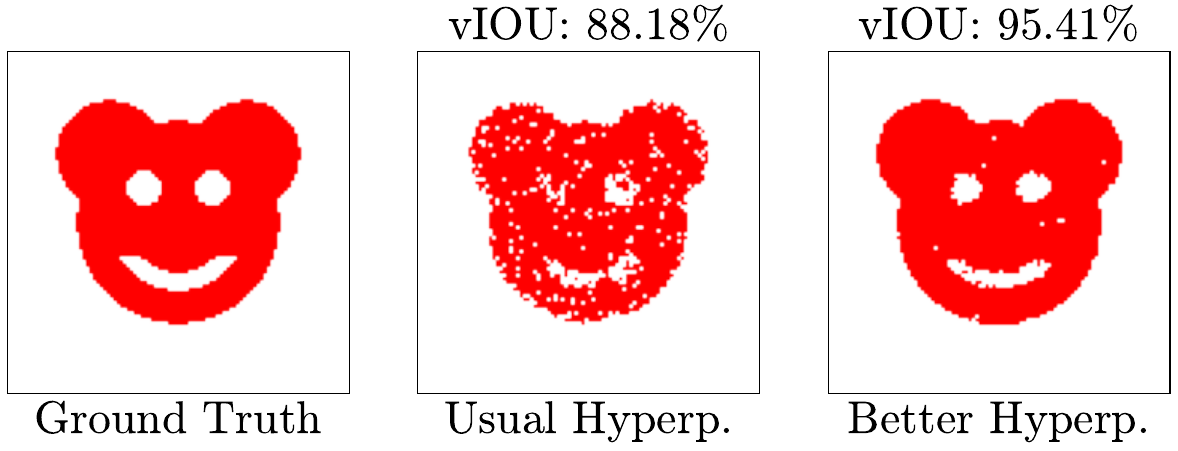}
    \caption{From left to right: The ground truth of our toy dataset, the area of a Polyra Swarm with default hyperparameters fitted to samples following our toy dataset, and the same with better chosen hyperparameters.}
    \label{fig:inihyper}
\end{figure}

We will now discuss reasonable ranges of each hyperparameter by varying only one parameter in each subsection. For every Polyra Swarm, we state the volume of the learned shape, aswell as the vIOU quality compared to the ground truth. We also plot the learned area for one indicative hyperparameter value.

Every other hyperparameter is kept at its default parameter (Table~\ref{tab:h:default}). Still, of course, it is often useful to vary hyperparameters in tandem (for example one hyperparameter that increases the volume and one hyperparameter that decreases it). An example of this is studied in Appendix~\ref{app:linearisation}.

\begin{table}[h]
\centering
\begin{tabular}{|lr|}
\hline
\textbf{Hyperparameter} & \textbf{Value} \\
\hline
Adim      & $2$   \\
Bdim      & $2$   \\
Extend      & $0$   \\
Minpoi      & $0$   \\
Quantile      & $0$   \\
Subsample      & $0$   \\
Sample count     & $10000$  \\
Model count       & $1000$   \\
\hline
\end{tabular}
\caption{Default hyperparameter settings used in every experiment unless mentioned differently.}
\label{tab:h:default}
\end{table}

\subsection{Adim}\label{app:h:adim}
The first hyperparameter we study is the number of conditions in the condition polytope ($A$ in Equation~\ref{eqn:submodel}). To truly approximate every possible shape, we require $Adim=dim+1$ (with the dimensionality of the data $dim$), as seen in Appendix~\ref{app:usaproof}. Still, most shapes can already be approximated well with a lower value of $Adim$. However, as Figure~\ref{fig:h:adim} shows, holes in the learned shape become impossible when $Adim$ is chosen too small. Additionally, it can also be useful to choose $Adim$ higher to make the learning algorithm focus more on small features. We do this in the the fit shown in Figure~\ref{fig:eye}.

\begin{figure}[H]
    \centering
    \includegraphics[width=0.6\linewidth]{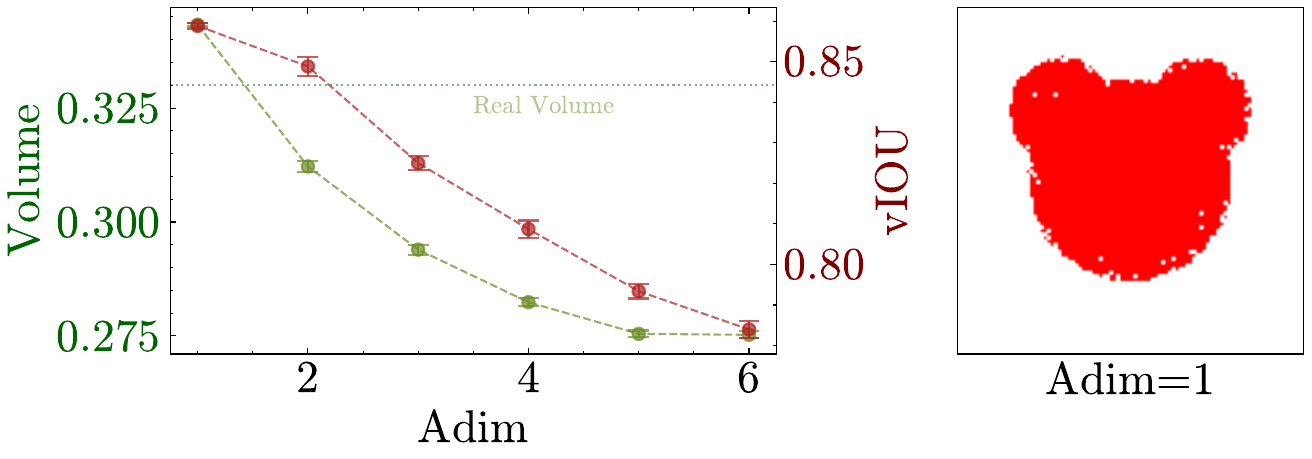}
    \caption{Hyperparameter influence of the $Adim$ hyperparameter.}
    \label{fig:h:adim}
\end{figure}

\subsection{Bdim}\label{app:h:bdim}
Next, we consider the number of conditions in the consequent polytope in each submodel ($B$ in Equation~\ref{eqn:submodel}) in Figure~\ref{fig:h:bdim}. This parameter is significantly less important than $Adim$, and serves more as a potential speedup. Effectively, it can be seen as a cheaper way of increasing the number of effective submodels the swarm uses. Compare for this Figure~\ref{fig:h:bdim} and Figure~\ref{fig:h:count}.

\begin{figure}[H]
    \centering
    \includegraphics[width=0.6\linewidth]{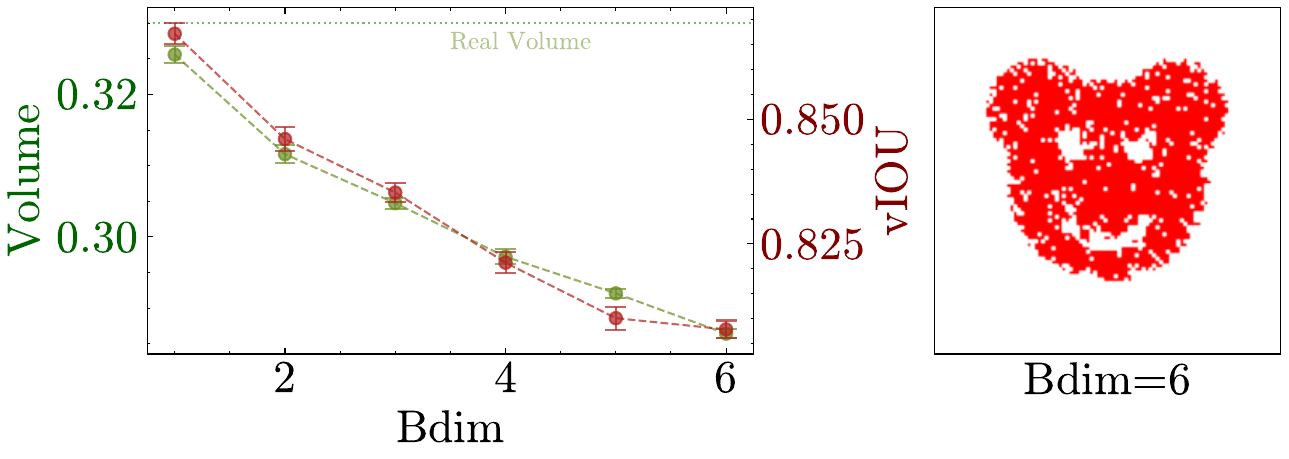}
    \caption{Hyperparameter influence of the $Bdim$ hyperparameter.}
    \label{fig:h:bdim}
\end{figure}

\subsection{Extend}\label{app:h:extend}
To learn the consequent region ($B_i$ in Equation~\ref{eqn:submodel}), we select the constraint region $(b_B)_i$ as the most extreme values allowed in the condition polytope (See Equation~\ref{eqn:maximaB}). However, it is unlikely that the most extreme value possible in each direction has already been observed in the training set $X_{train}$. Thus, we introduce a hyperparameter to extend the range. In total, the range of allowed values is evenly extended by a factor $1+extend$. So given values observed between $0$ and $1$ with an extend of $0.2$, we would allow values between $-0.1$ and $1.1$ in a given submodel.

As visible in Figure~\ref{fig:h:extend}, increasing this parameter can help reduce the fragmentation of the learned region. And while a higher value of $extend$ also increases the volume, there is a sweet spot between false positives and false negatives.

\begin{figure}[H]
    \centering
    \includegraphics[width=0.6\linewidth]{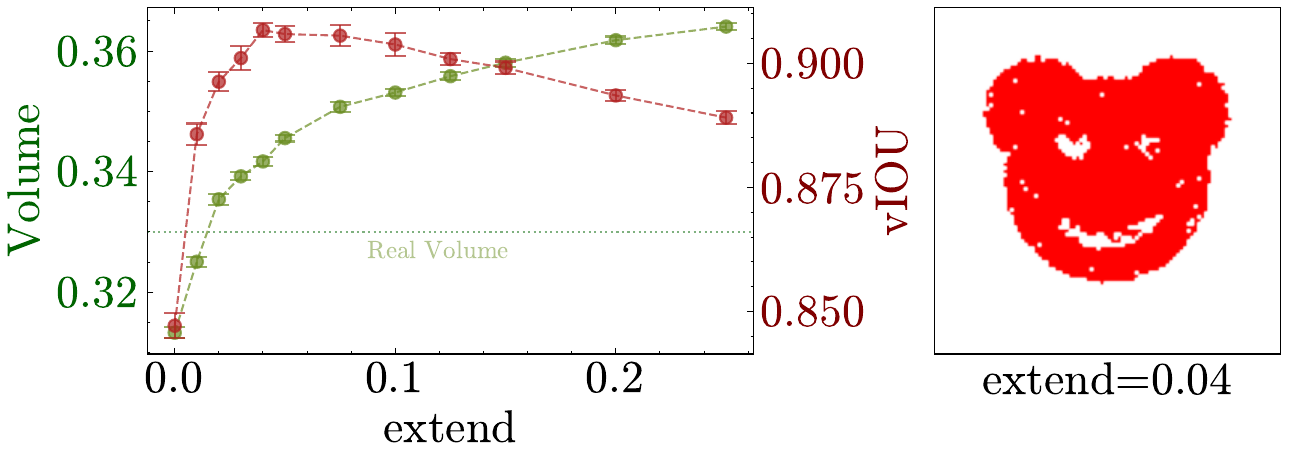}
    \caption{Hyperparameter influence of the $extend$ hyperparameter.}
    \label{fig:h:extend}
\end{figure}

\subsection{Minpoi}\label{app:h:minpoi}
An alternative to reduce the fragmentation of the learned shape is to demand that a minimum number of samples ($minpoi$) are included in the condition polytope. This removes small condition polytopes, which tend to lead to highly fragmented shapes. This can be seen as a hyperparameter controlling the convexity of the learned shape. If we demand that every sample is included in each condition polytope, it would become trivial and we would only learn the consequent polytopes, resulting in a convex shape. This can be seen in Figure~\ref{fig:h:minpoi}, where a high value of $minpoi$ results in the mouth and eyes not being learned.

\begin{figure}[H]
    \centering
    \includegraphics[width=0.6\linewidth]{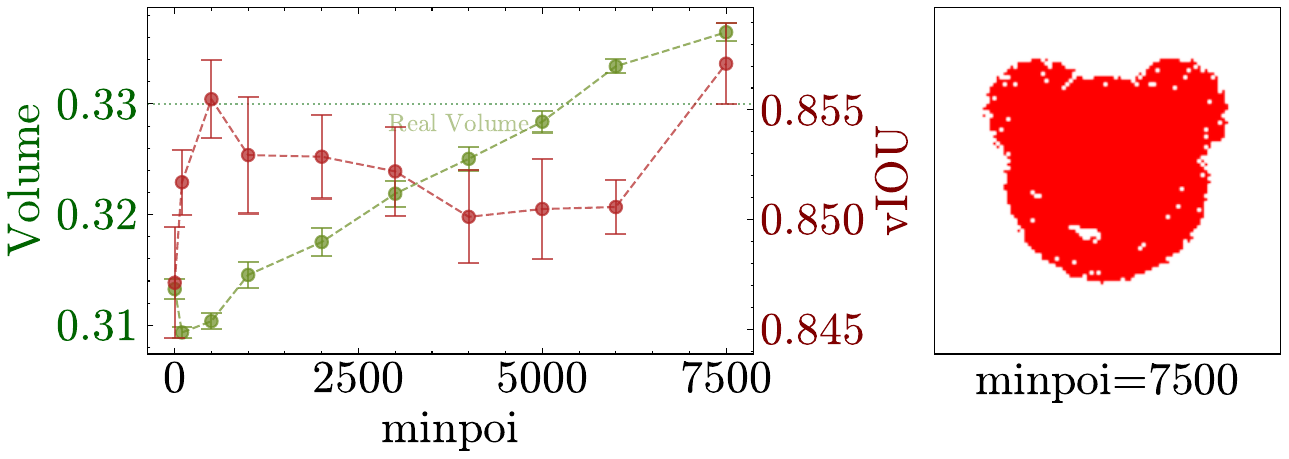}
    \caption{Hyperparameter influence of the $minpoi$ hyperparameter.}
    \label{fig:h:minpoi}
\end{figure}

\subsection{Quantile}\label{app:h:quantile}
The opposite effect of $extend$ is achieved with the $quantile$ parameter. A $quantile>0$ means that only a fraction of $1-quantile$ is used to select the most extreme values in Equation~\ref{eqn:maximaB}. This hyperparameter breaks the assumption that every training sample is considered inside the learned region (Appendix~\ref{app:alltrueproof}). Furthermore, as Figure~\ref{fig:h:quantile} shows, this parameter increases the fragmentation of the learned shape.
This parameter is most useful, when also varying other hyperparameters. See for this Appendix~\ref{app:linearisation}.
\begin{figure}[H]
    \centering
    \includegraphics[width=0.6\linewidth]{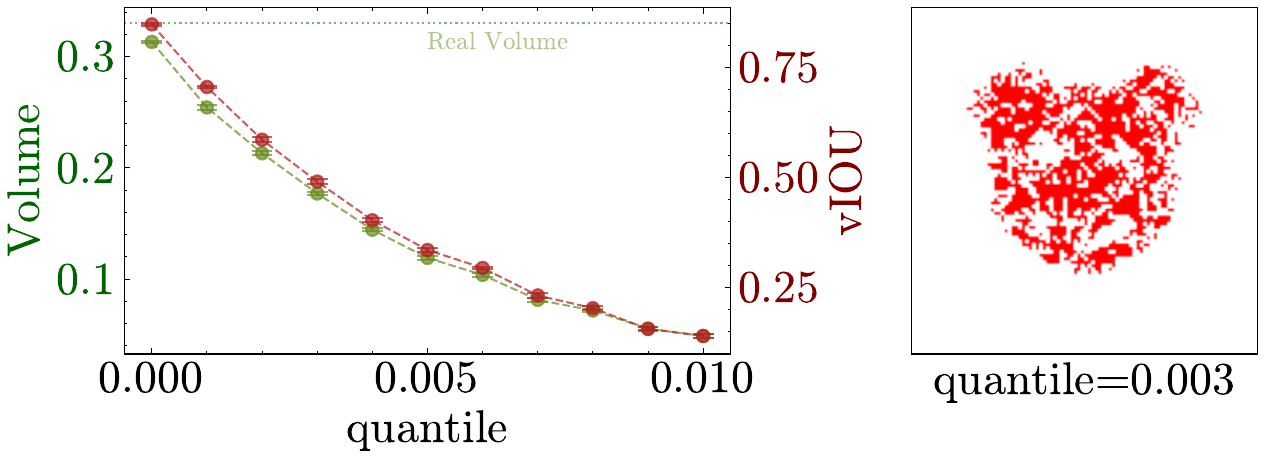}
    \caption{Hyperparameter influence of the $quantile$ hyperparameter.}
    \label{fig:h:quantile}
\end{figure}

\subsection{Subsample}\label{app:h:subsample}
A classical approach for one-class classification/outlier detection ensemble methods is to use subsampling~\cite{subsamp} (See Figure~\ref{fig:h:subsample}). Here every submodel is trained only on $1-subsample$ of the training samples. This increases fragmentation and also breaks the assumption that every training sample is considered inside the learned shape. So far, the only benefit of this parameter we have found is the potential speedup, but we include it here because of potential comparisons to outlier ensembles.
\begin{figure}[H]
    \centering
    \includegraphics[width=0.6\linewidth]{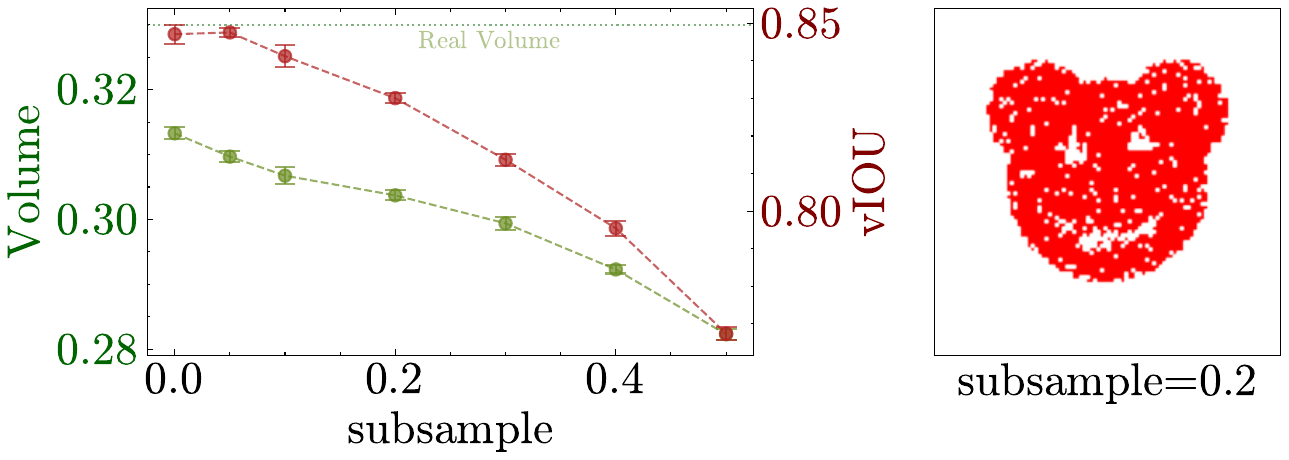}
    \caption{Hyperparameter influence of the $subsample$ hyperparameter.}
    \label{fig:h:subsample}
\end{figure}

\subsection{Dataset Size}\label{app:h:datasize}
While not strictly a hyperparameter, we also study the effect a different training set has. As Figure~\ref{fig:h:datasize} shows, a larger data size generally increases both the volume and, at the same time also, the performance. This implies that by reducing the dataset size, we increase primarily the number of false negatives and, thus, the fragmentation. In the extreme, we have so few data points that only these are considered inside of the shape, and the volume goes to zero from $O(100)$ samples. However, as we can still fit Polyra Swarms to small datasets (e.g. Figure~\ref{fig:rw} with $272$ samples), this is controllable with various hyperparameters.

\begin{figure}[H]
    \centering
    \includegraphics[width=0.6\linewidth]{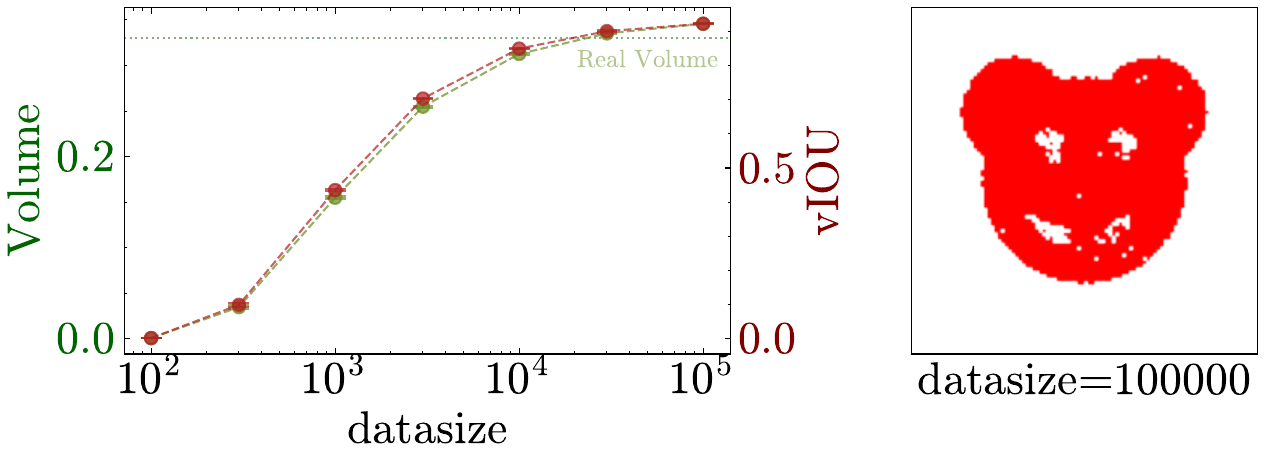}
    \caption{Influence of the dataset size/sample count.}
    \label{fig:h:datasize}
\end{figure}

\subsection{Ensemble Size}\label{app:hyper:ensemblesize}
The most important parameter to change for small datasets is the number of submodel equations used. Since every equation cuts of parts of the input space, the volume decreases with larger ensembles. This is shown in Figure~\ref{fig:h:datasize}. Thus, to fit Polyra Swarms to smaller datasets, we use fewer submodels, and to larger datasets (Figure~\ref{fig:eye} uses one million submodels), we can and should use more.

\begin{figure}[H]
    \centering
    \includegraphics[width=0.6\linewidth]{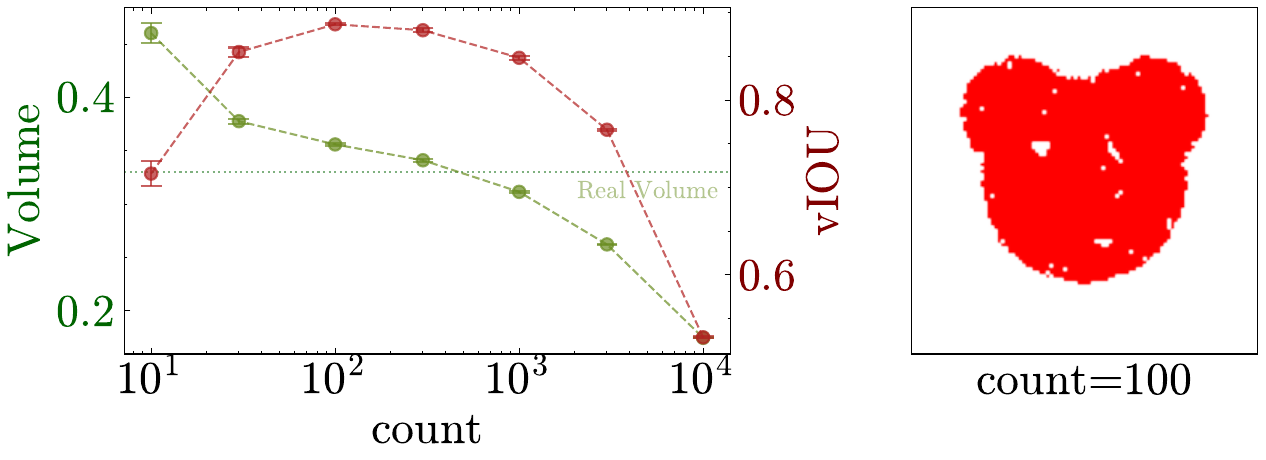}
    \caption{Influence of the number of submodels.}
    \label{fig:h:count}
\end{figure}

\section{Linearisation}\label{app:linearisation}
In the previous appendix, we mainly studied the effect singular hyperparameter changes have on a learned model. However, interesting effects also happen when we combine multiple hyperparameters. In Figure~\ref{fig:linearisation}, we summarize how we can use this to achieve an effect that we call \emph{linearisation}. This effect uses both the $quantiles$ and the $extend$ hyperparameter to, at the same time, cut off extreme samples in long tail distributions and extend the area canceling out this effect.
\begin{figure}[H]
    \centering
    \includegraphics[width=0.6\linewidth]{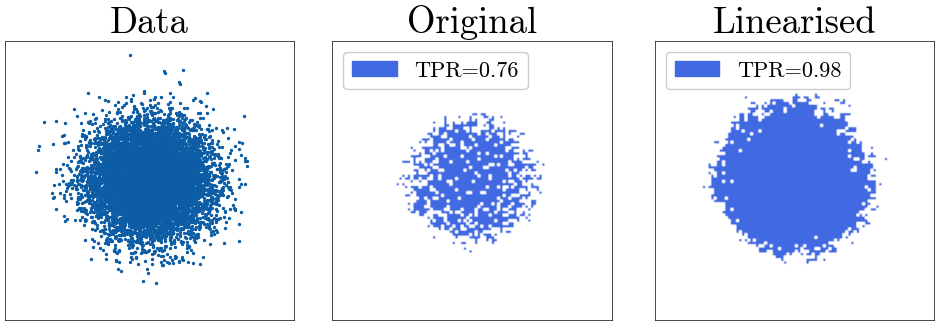}
    \caption{Example of a hyperparameter combination that we name \emph{linearisation}. When approximating the shape of a distribution with long tails, like the normal distribution shown on the left, the low likelihood of points further away from the mean results in a highly fractured distribution in low-density regions. This can be seen in the fit with default hyperparameters in the middle. Only $76\%$ of points following the same normal distribution are considered inside of the shape. Instead, we consider a combination of the quantile hyperparameter ($quantiles=0.025$) and the extended parameter ($extend=0.2$). While the quantiles hyperparameter removes noise in the low-density region of the distribution, the extend hyperparameter extends each model to still include the same range. This results in the much better fitting shape on the right side, which considers $98\%$ of normal samples inside the shape. Notice that because of the infinite tails of the gaussian distribution, we can not aim for $100\%$ here.}
    \label{fig:linearisation}
\end{figure}

\section{Polyra Name Inspiration}\label{app:name}
Each submodel of our ensemble cuts away a part of the input space until only the desired shape remains. This is very similar to how we imagine piranhas to eat a much larger animal. Thus, we decided to call our method a portmanteau word between "piranha" and the "polytope"s we use. Thus we call our models "Polyra Swarms".

\section{ROC-AUC bias towards Anomaly Scores instead of Binary Decisions}\label{app:rocbias}
When comparing our Polyra Swarms as anomaly detection algorithms, there is a fundamental difference to most other anomaly detection algorithms. Instead of the binary decisions provided by Polyra Swarms, these algorithms usually generate continuous anomaly scores (where a higher anomaly score implies a sample to be more anomalous)~\cite{metasurvey}. As we show in Figure~\ref{fig:rocbias}, this generally increases the ROC-AUC quality such a method achieves, even when the method used is the exact same.

\begin{figure}[htbp]
    \centering
    \includegraphics[width=0.8\linewidth]{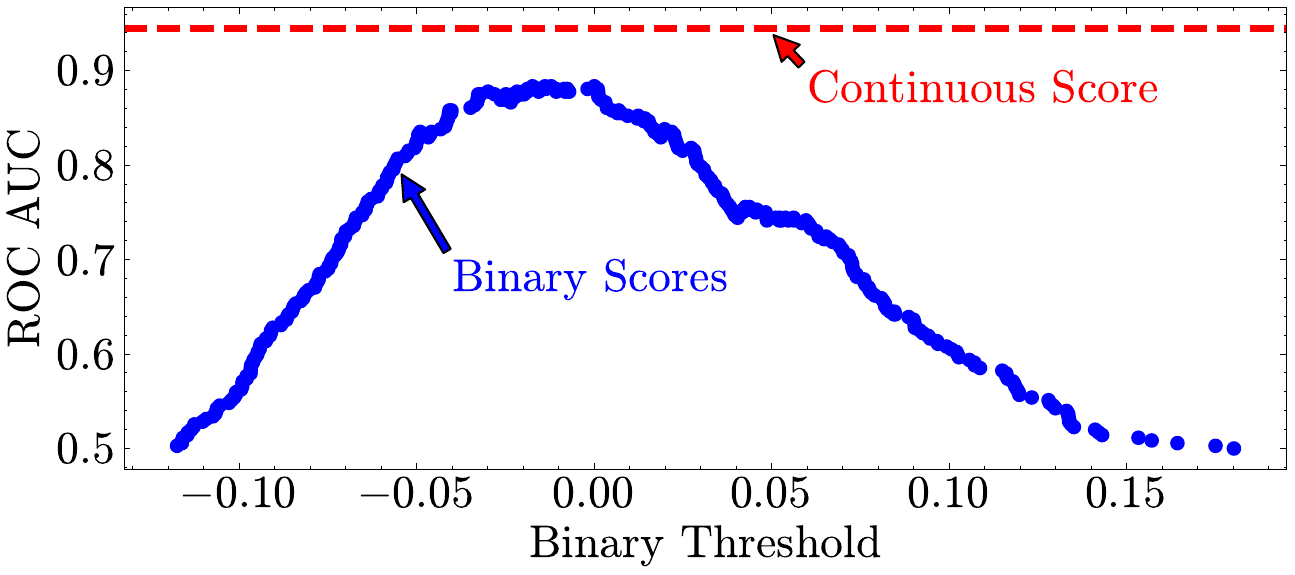}
    \caption{ROC AUC score on the cardio dataset~\cite{cardioDS} with an isolation forest. In red, you see the ROC AUC score when using continuous scores, while the blue curve shows the performance with binary scores (using every possible threshold). The continuous score is at least $6\%$ higher than any binary one, showing that we can not fairly compare continuous scores to binary ones.}    \label{fig:rocbias}
\end{figure}

\section{MNIST Classification}\label{app:mnistclass}

We further extend our classification experiments to higher dimensional MNIST data. Because their distribution is more complicated than the toy dataset in Section~\ref{sec:class}, we generally observe four outcomes for a binary classifier:
\begin{itemize}
    \item True: The sample is classified as the correct class (and only as the correct class).
    \item False: The sample is classified as the wrong class (and only as the wrong class).
    \item Outlier: The sample is neither classified as the right or wrong class.
    \item Overlap: The sample is classified as both the right and wrong class.
\end{itemize}

\begin{table}[h]
\centering
\begin{tabular}{|l|cccc|}
\toprule
\textbf{Model Name} & \textbf{True} & \textbf{False} & \textbf{Outlier} & \textbf{Overlap} \\
\midrule
Direct Model       & $0.8270$ & $0.0014$ & $0.0845$ & $0.0842$ \\
PCA (2 components) & $0.9877$ & $0.0019$ & $0.0080$ & $0.0024$ \\
Anomaly Scores     & $0.9943$ & $0.0024$ & - & $0.0033$ \\
\bottomrule
\end{tabular}
\caption{Performance of three different applications of Polyra Swarm classifiers to the binary MNIST classification task ($0$ vs $1$). We compare four different situations: Either the model correctly classifies a sample, the model incorrectly classifies a sample, the model considers a sample to be neither class or the model could see a sample as either class. Additionally, we consider three different binary classifier. The first one simply trains a Polyra Swarm for each class, the second one does the same, but after applying a pca transformation to reduce the dimensionality. The third model uses not the binary decision of whether a sample is in a class but the anomaly scores as used in Chapter~\ref{sec:ad} to consider where a sample is more likely to match. This makes outliers impossible by construction.}
\label{tab:mnistclass}
\end{table}

We show the fraction of each situation in Table~\ref{tab:mnistclass}. Simply applying polyra classifier to the raw image data is not very effective. While situations where the classifier directly classifies a test sample as wrong are very rare, both the situations $Outlier$ and $Overlap$ happen in about $8\%$ of cases. We state our hyperparameters in Table~\ref{tab:h:zerosgen}, and using various hyperparameters we can change the trade-off between $Outlier$ and $Overlap$ cases (for example using $extend$). However, using a PCA algorithm to remove every feature except the two most important principle components, we can achieve significantly better results, with all error cases much lower than $1\%$. This is visualized in Figure~\ref{fig:binarypca}.

We interpret this effect as indicator that Polyra Swarms don't work well in high-dimensional data but still have the precision to learn complicated shapes (See Chapter~\ref{sec:c:highdim}). We can further improve the performance, by considering anomaly scores (See Chapter~\ref{sec:ad}) in the case of outliers. The class that considers a sample as less anomalous gets assigned to each sample. This roughly halves the overall non-True cases, but also removes the reject option of our classifier.

\begin{figure}[htbp]
    \centering
    \includegraphics[width=0.8\linewidth]{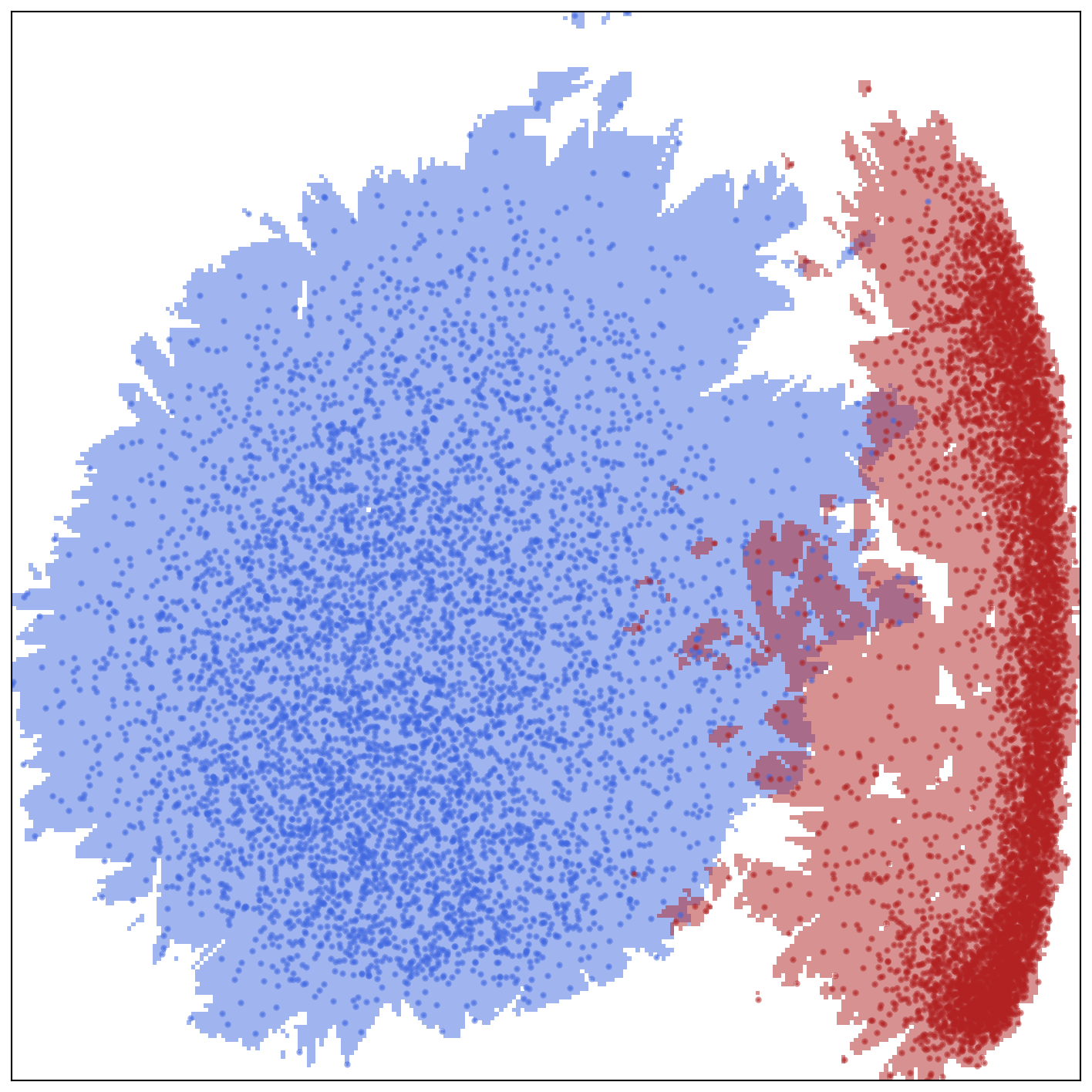}
    \caption{Distribution of binary classification task ($0$ vs $1$ on MNIST) in a two-dimensional principal component space and learned regions by a Polyra Swarm classifier. }
    \label{fig:binarypca}
\end{figure}

\begin{table}[h]
\centering
\begin{tabular}{|lr|}
\hline
\textbf{Hyperparameter} & \textbf{Value} \\
\hline
Sample count     & $5923+6742$  \\
minpoi           & $500$  \\
extend           & $0.1$  \\
PCA components   & $2$ \\
\hline
\end{tabular}
\caption{Non-default hyperparameter settings used in the mnist classification experiment.}
\label{tab:h:zerosgen}
\end{table}

\section{Abstraction Procedures}\label{app:abs}

We are interested in such algorithms that transform a Polyra Swarm into an approximately similar form while drastically reducing its complexity. Since we use a tree-based structure to represent a learned Polyra Swarm, we consider the number of nodes $C$ needed to describe it as a measure of complexity.

Generally, we differentiate between two groups of abstraction procedures. The first one consists of simple transformations that slightly reduce the complexity $C\rightarrow C^{\prime} \gtrsim \frac{C}{2}$, while not noticeably altering the prediction of the Polyra Swarm. These simplifications are generally extremely fast, and will be executed commonly to transform the Polyra Swarm tree in a uniform format. These will be presented in the following subsection~\ref{app:abs:simp}. The second, maybe more interesting group of abstraction procedures do change the learned prediction of a Polyra Swarm, but reduce the complexity significantly more $C\rightarrow C^{\prime} \ll C$. These algorithms are much more complicated and we will study multiple different algorithms with slightly different approaches and use-cases in the following Appendices (\ref{app:abs:1d}, \ref{app:abs:sample}, \ref{app:abs:lp}). All of these algorithms follow the same general structure, which will be explained in Appendix~\ref{app:abs:gen}.

\subsection{Simplification}\label{app:abs:simp}

Our simplification procedure generally consists of various transformations that will be applied to transform a Polyra Swarm in tree representation into a uniform form. We will list these transformations in this appendix to increase reproducibility, sorted by the type of leaf they are applied to.

\subsubsection{And-Leaves}
\begin{equation}
    \TRUE \land \cdots \rightarrow \cdots
\end{equation}
\begin{equation}
    \FALSE \land \cdots \rightarrow \FALSE
\end{equation}
\begin{equation}
    a \land (b\land c) \rightarrow a \land b \land c
\end{equation}
\subsubsection{Or-Leaves}
\begin{equation}
    \TRUE \lor \cdots \rightarrow \TRUE
\end{equation}
\begin{equation}
    \FALSE \lor \cdots \rightarrow \cdots
\end{equation}
\begin{equation}
    a \lor (b\lor c) \rightarrow a \lor b \lor c
\end{equation}
\subsubsection{Not-Leaves}
\begin{equation}
    \neg (\mathrm{A}\cdot x \leq b) \rightarrow -\mathrm{A}\cdot x \leq -b -\epsilon
\end{equation}
This equation allows us to remove all Not-Leaves from our Polyra Swarm Tree. $\epsilon$ represents a small constant and is chosen as $\epsilon=10^{-6}$.
\begin{equation}
    \neg (a \lor b) \rightarrow \neg a \land \neg b
\end{equation}
\begin{equation}
    \neg (a \land b) \rightarrow \neg a \lor \neg b
\end{equation}
\begin{equation}
    \neg\neg a \rightarrow a
\end{equation}

\subsection{General Abstraction}\label{app:abs:gen}

The above-described simplification procedure transforms a Polyra Swarm tree into a format that looks similar to:
\begin{equation}
    (\mathrm{Halfspace} \vee (\mathrm{Halfspace}\wedge\mathrm{Halfspace}))\wedge(\mathrm{Halfspace} \vee (\mathrm{Halfspace}\wedge\mathrm{Halfspace}))\wedge\cdots
    \label{eqn:abs:before}
\end{equation}

Still, many of these terms have very similar effects, and it is possible to reduce the amount of these terms significantly using the abstraction algorithms we describe in the following appendices. All of the abstraction algorithms that we study here follow the same pattern. We first convert the tree into a format similar to conjunctive normal form (OR Terms below a AND leaf), and transform this form into disjunctive normal form (AND terms below an OR leaf). 

These transformations rely on well-known logical transformations. Namely
\begin{equation}
    (a \vee b) \wedge (c\vee d) \rightarrow (a\wedge c) \vee (a\wedge d) \vee (b\wedge c) \vee (a\wedge d) 
\end{equation}
and
\begin{equation}
    (a \wedge b) \vee (c\wedge d) \rightarrow (a\vee c) \wedge (a\vee d) \wedge (b\vee c) \wedge (a\vee d) 
\end{equation}

The first transformation is generally fast since the initial, learned tree in Equation~\ref{eqn:abs:before} is already very similar to the desired conjunctive normal form and we thus achieve a linear runtime as a function of submodels used. However, transforming a large logical tree in conjunctive form (with $C_1$ OR leaves containing $C_2$ halfspace conditions each) into the equivalent tree in disjunctive normal form generally produces a tree containing $C_2^{C_1}$ different AND leaves. This is generally infeasible for any swarm with more than a few submodels. However, if it would be possible, many of these AND leaves are trivial (e.g. $a \wedge \neg a$) or duplicates of each other (e.g. $a \vee a$), possibly resulting in less complexity than the original swarm tree.

To circumvent this, we apply our abstraction procedure iteratively.

Instead of 
\begin{equation}
    q=\mathrm{ABSTRACT}(\mathrm{DNF}(t_1 \land t_2 \land \cdots \land t_N))
    \label{eqn:abs:before}
\end{equation}
We apply

\begin{equation}
    q_i=\mathrm{ABSTRACT}(\mathrm{DNF}(q_{i-1},t_i)), \; q_1=t_1, q=q_N
    \label{eqn:abs:after}
\end{equation}

This removes the exponential time cost of Equation~\ref{eqn:abs:before} when the ABSTRACT procedure is powerful enough to guarantee that $C(q_i)\leq C_{max}$ regardless of the complexity of $C(\bigwedge t_i)$. Instead, we have an exponential dependency on ($C_{max}$) since we need to transform an equation containing up to $C_{max}$ terms into a normal disjunctive form for each initial AND leaf. This exponential dependency can be surprisingly useful, since it means that the speed of the algorithm is already an indicator for the resulting complexity reduction of the abstraction. We generally want to build our abstraction procedures and hyperparameters so that $C(q_i)$ converges:

\begin{equation}
    \lim_{i\rightarrow\infty} C(q_i)=C(q) < \infty
    \label{eqn:abs:condition}
\end{equation}

We will now study various ABSTRACT procedures in the following Appendices.

\section{One-dimensional Abstraction: The Rangefinder Algorithm}\label{app:abs:1d}

We begin by assuming that our input space is one-dimensional. This simplifies creating an abstraction algorithm immensely, as the largest difficulty in creating an abstraction algorithm is to make sure that $C(q_i)\leq C_{max}$ is fulfilled, and more submodels do not require a more complicated description.

In one dimension, every halfspace condition can be represented as either $x\ge a$ or $x\le b$. Thus, every possible shape can be represented as $\bigvee_i (a_i\leq x \leq b_i)$.
Further, we can always find a representation, where there is no overlap between conditions ($a_{i+1}>b_i\;\forall i$), since when two conditions overlap, we can combine them together.
\begin{equation}
    (a\leq x \leq b) \vee (c\leq x \leq d) \textrm{ ,where } (c\le b \land a\le d) = a\leq x \leq d
\end{equation}

Thus, the resulting maximum complexity is bounded by $C_{max}=a\cdot (1+\mathrm{Holes)}$ with a small constant $a$ and the number of holes learned (Holes), guaranteeing effective abstraction.

In practice, $\textrm{ABSTRACT}_\textrm{1d}$ consists of two steps. We first consider each OR leaf and transform it into $a\leq x \lor x\leq b$ by considering the most extreme cases ($a\le x \lor b\le x,\;a<b\rightarrow a\le x$). If $b<a$, we remove the condition and replace it with $\TRUE$. Afterward, we follow the iterative conversion procedure converting to disjunctive normal form and apply the equivalent transformation in each AND leaf ($a\le x \land b\le x,\;a<b\rightarrow b\le x$).

This abstraction procedure is significantly faster than the algorithms studied in the following appendices and, in comparison to the following algorithms, does not need to alter the prediction of a Polyra Swarm and gives us a theoretical solution on the abstraction benefits. However, it only works on one-dimensional swarms. Its current use case in this paper is somewhat hidden. We use it to power the generative model (Appendix~\ref{app:generative}) and for precise range predictions (Sections~\ref{sec:reg} and \ref{sec:match}).

\section{N-dimensional Abstraction: Sampling Abstraction}\label{app:abs:sample}

In more than one dimension, we are no longer able to reduce every AND or OR leaf into such a simple uniform representation. And while there are still many simplifications that can be made, to guarantee that $C(q_i)$ is bounded we need to change the predictions of our Model (slightly). As we show in the main paper, this is often beneficial since it counters the fragmentation of our swarm and thus reduces overfitting effects. These benefits of reduction are a well-studied effect in machine learning and philosophy~\cite{occam1,occam2,occam3}. 

Our procedure $\textrm{ABSTRACT}_\textrm{sample}$ consists of three steps.

\subsection{Redundant ANDs}\label{app:abs:sample:AND}
This part considers each AND-like leaf, and thus terms of the form $p=h^1 \land h^2 \land h^3\land \cdots$, where $h^i$ are halfspace conditions. We first check if $\exists x\in \mathbb{R}^d$ so that $p(x)=\TRUE$. If not, we can discard the leaf ($p=\FALSE$). 
If we find a single sample that makes this AND-like leaf nontrivial, we consider removing single halfspace conditions and check whether $\exists x\in \mathbb{R}^d$ so that $ (h^2\land h^3\land \cdots)(x)=\TRUE$ but not $h^1(x)=\FALSE$. If none such a sample exists, the condition $h^1$ provides no benefit and is removed.
We iterate until there are no more conditions to be removed.

Our implementation here requires random samples to check whether there exists a sample that is fulfilling certain conditions. This dependency on random samples incurs an effectively exponential dependency on the dimensionality~\cite{curseOdim}, and thus, sample abstraction can only be applied to low-dimensional tasks. We consider a better scaling version of this algorithm in Appendix~\ref{app:abs:lp}.

\subsection{Redundant ORs}\label{app:abs:sample:OR}
Next, we consider the OR-like leaf, and thus terms of the form $s=p_1 \lor p_2 \lor p_3 \lor \cdots$. Instead of random samples, we consider the set of points we train on $X$. We consider each child term and check whether $(p_2\lor p_3\lor \cdots)(x)=\TRUE \; \forall x\in X$. If this is the case, the term $p_1$ does not provide any benefit and can be removed.

\subsection{Puzzling}\label{app:abs:sample:PUZZ}
Finally, we will consider the top-level OR-like leaf again and see if we can merge different terms $p_i$ together. Consider for example:
\begin{equation}
    s=(p_1 \land p_2)\lor (p_1\land \neg p_2) =p_1
    \label{eqn:abs:example}
\end{equation}
For this, we consider for each set of two $p_i=h_i^1 \land h_i^2 \land h_i^3 \land \cdots$ if there is a description that matches both. For this, we exploit that when merging $p_a$ and $p_b$, the resulting polytope $p_m$ needs to contain only halfspaces that do not contradict either $h_a^j$ or $h_b^j$. Thus, our candidate for the halfspaces describing $p_m$ are those of $p_a$ that do not contradict $p_b$ and those of $p_b$ that do not contradict $p_a$. This procedure creates a polytope encompassing both initial polytopes and already solves the example in Equation~\ref{eqn:abs:example}. However, it also suggests a hull when the polytopes do not match. So, we need to check whether the resulting volume is larger than the volume of the initial polytopes $\mathrm{Vol}(p_m)>Vol(p_a\lor p_b)$ and reject the combination in this case. To calculate the volume, we again use random samples. We iterate this procedure again until there are no more polytopes to merge.

So far, no step in the algorithm has altered the prediction of the Polyra Swarm (if enough samples are used for the sampling-based decisions). However, in our experiments, running this algorithm violates the condition that $C(q_i)$ is bounded, as the abstraction algorithm requires more and more time in each step. To solve this, we allow the algorithm to approximate slightly by introducing a hyperparameter $\delta V$. Using this hyperparameter, we allow the merge of two shapes to be slightly bigger than the initial volume $\frac{\mathrm{Vol}(p_m)}{Vol(p_a\lor p_b)}\le 1+\delta V$. This allows the abstraction algorithm to remove small regions inside the shape that do not match the expectation, thus preventing fragmentation effects. In our experiments, setting $\delta V=5\%$ seems to work well.

\section{Abstraction through Linear Programming}\label{app:abs:lp}
Many of the decisions of the abstraction algorithm introduced in Appendix~\ref{app:abs:sample} depend on random samples and thus generally might not scale well to high-dimensional tasks. However, we find that large parts of the abstraction procedure can be replaced by linear programming~\cite{linearprogramming}.

Linear programming is a mathematical method for optimizing a linear objective function, subject to a set of linear equality and inequality constraints. Most of the sampling steps in Appendix~\ref{app:abs:sample} can be formulated into such a linear program checking whether a set of linear inequality constraints has a solution. We will discuss here quickly how to do so aswell as the limitations of this approach.

\subsection{Trivial Subterms}
The first step in Section~\ref{app:abs:sample:AND}, checks whether one polytope has a single solution. This can be directly expressed as a linear program, with the halfspace conditions providing the constraints.

\subsection{Redundant AND-Terms}
Similarly, our check in Section~\ref{app:abs:sample:AND} for redundant terms iterates over checks that can be expressed as a linear program. Here $h^2\cdots h^N$ represent constraints, and by inverting $\neg h^1$ we can also consider this as a further halfspace condition. Notice that the last step requires carefully choosing an $\epsilon>0$, as linear programs generally assume non-strict constraints. This is less important when using sampling, as it is very unlikely that one equation is randomly exactly fulfilled.

\subsection{Redundant OR-Terms}
Our method in Section~\ref{app:abs:sample:OR} relies not on random samples but on the training data. Since the amount of samples does not increase with the dimension, we will keep it as is.

\subsection{Puzzling}
The difficult part of our linear programming-based approach is the approximation. While all decisions before have been of the type $\exists x \textrm{ so that }\cdots$, our condition in Section~\ref{app:abs:sample:PUZZ} checks whether there are too many samples $x$ that are both not in $p_1$ and $p_2$ but are in $s$. We were not able to solve this using linear programming. Instead, we have to rely on sampling again. Still, we can do better than simply drawing samples from a uniform space. For this, consider that every point of interest should lie in $s$ and that we accept a merged description when this polytope contains at most a few samples that are not in $p_1$ and $p_2$. Thus, we can rely on random samples generated from the known polytope $s$. This is a well-studied research area~\cite{polytopesampling} with generally less than exponential time cost. Now we just have to check whether these samples are in $p_1$ or $p_2$.

\subsection{Time complexity}
We use sampling-based abstraction in this paper since it is generally more efficient for the two-dimensional case, and we cannot visualize higher-dimensional shapes well. However, because sampling random points generally might require exponentially many samples in the dimension $dim$, this approach is less valuable for high-dimensional tasks. In our experiments, the limit seems to be roughly $dim=6$ on a consumer-grade laptop in a reasonable time.

To fix this, we suggest using linear programming. The runtime of this depends on the runtime of the algorithm used to solve linear programs, with some linear programs having time complexity $O((dim+k)^{3.5})$\cite{interiorPointMethodsLinearOptimization}, where $k$ is the number of constraints. Additionally, the number of such checks depends on the hyperparameters chosen but is also generally polynomial in $dim$, resulting in a polynomial runtime.

Still, even our linear programming abstraction algorithm is very likely also not the best possible abstraction algorithm and we look forward to the results of further research.


\section{Clustering through Abstraction}\label{app:clustering}

When considering Figure~\ref{fig:rw} in the main paper, our abstraction process successfully finds the two well-known clusters in the Old-Faithful dataset~\cite{oldfaithful}. Here, we want to explore further whether we can use abstraction to find clusters in a dataset. This would be interesting since when we interpret our abstraction algorithm in this way, it generates very efficient, small, human-readable descriptions for each cluster. One large restriction will be that each abstracted Polyra Swarm has to be a set of polytopes connected by OR leaves, limiting the shape of a cluster to be convex.

We test this on a toy dataset in Figure~\ref{fig:cluster}. We generate data ($10k$ samples) in three uniform regions, train a $1000$ submodel Polyra Swarm with default parameters, and abstract it with the default $\delta V=5\%$. This generates the abstract description shown in Equation~\ref{eqn:cluster:res}. We show the three polytopes found in different colors in Figure~\ref{fig:cluster}, and these match very well with the initial squares (except for a small deviation in Cluster 3). This proves that there is potential in using polyra abstraction for clustering and in a more holistic view of machine learning models and we look forward to considering it further in future publications.

Please notice that the description of the clusters is still less than optimal. Six conditions describe the first and second clusters, and seven conditions describe the third cluster. As four conditions would be optimal for each cluster, this implies that the description could be abstracted further.

\begin{figure}[H]
    \centering
    \includegraphics[width=0.6\linewidth]{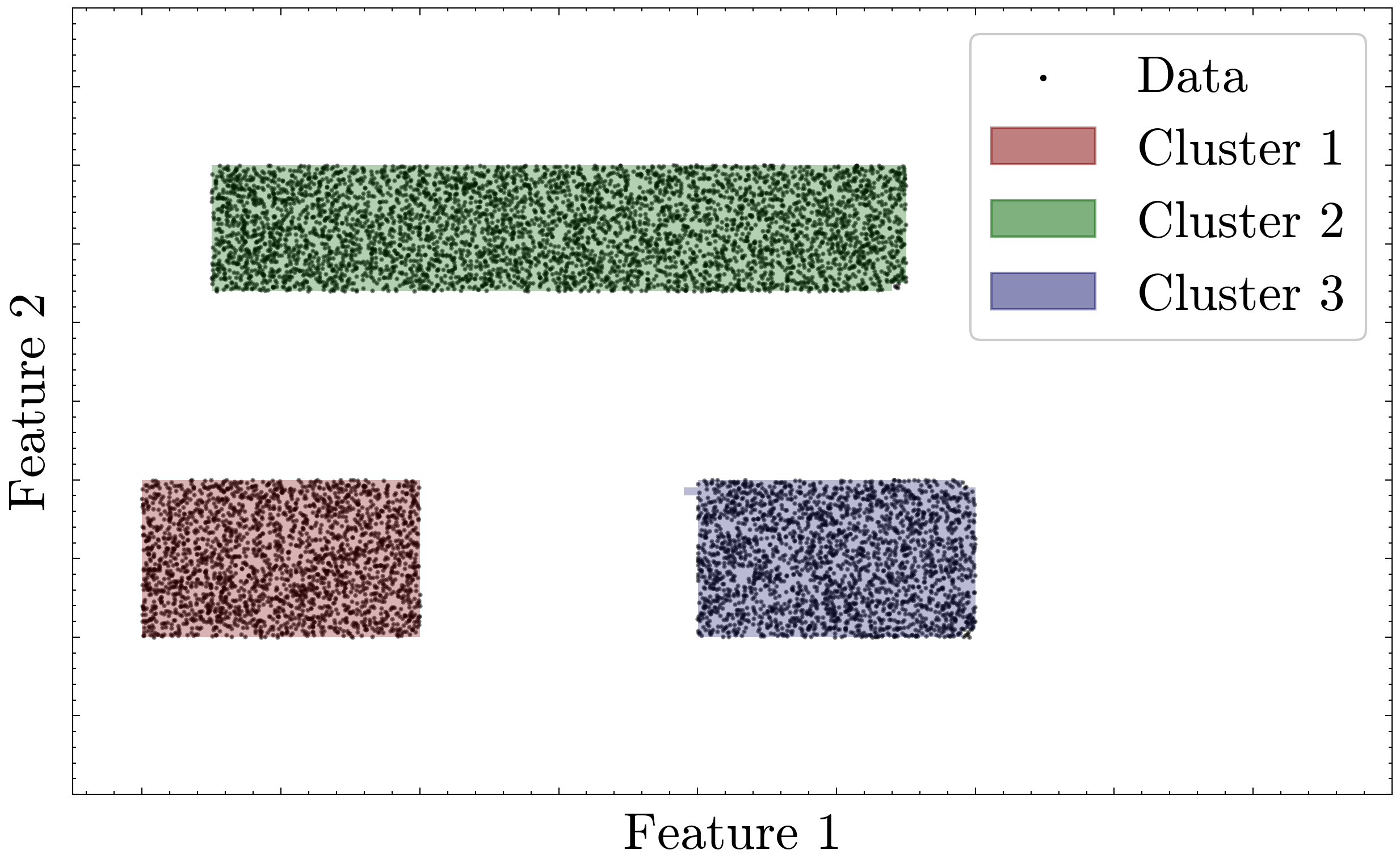}
    \caption{Toy experiment to show that abstraction can be used for clustering. We generate data following one of three square regions and fit a default hyperparameter Polyra Swarm to it. The abstraction of this swarm model is the disjunction of three polytopes (Equation~\ref{eqn:cluster:res}), which very well represent the three square regions.}
    \label{fig:cluster}
\end{figure}

\begin{equation}
\begin{aligned}
    &\left( \begin{bmatrix}
    -0.005 & -1.0\\ 
    -1.0 & 0.0054\\ 
    -0.0159 & 0.9999\\ 
    -0.9603 & 0.2789\\ 
    -1.0 & -0.0067\\ 
    1.0 & 0.0059
    \end{bmatrix}
    \cdot x \leq 
    \begin{bmatrix}
    -0.2011\\ 
    -0.198\\ 
    0.3962\\ 
    -0.0839\\ 
    -0.2018\\ 
    0.4015
    \end{bmatrix} \right)\\
\bigvee\quad 
    &\left( \begin{bmatrix}
    -0.007 & -1.0\\ 
    -1.0 & 0.0012\\ 
    -0.0099 & 1.0\\ 
    1.0 & -0.0004\\ 
    0.0099 & -1.0\\ 
    0.8172 & -0.5763
    \end{bmatrix}
    \cdot x \leq 
    \begin{bmatrix}
    -0.642\\ 
    -0.2493\\ 
    0.7968\\ 
    0.7497\\ 
    -0.6334\\ 
    0.2368
    \end{bmatrix} \right)\\
\bigvee\quad 
    &\left( \begin{bmatrix}
    -0.9997 & -0.0236\\ 
    -0.8234 & 0.5675\\ 
    0.0107 & 0.9999\\ 
    0.998 & 0.0624\\ 
    -0.0075 & -1.0\\ 
    0.0006 & -1.0\\ 
    1.0 & 0.0067\\ 
    0.9631 & -0.2691
    \end{bmatrix}
    \cdot x \leq 
    \begin{bmatrix}
    -0.6048\\ 
    -0.2704\\ 
    0.408\\ 
    0.8208\\ 
    -0.2051\\ 
    -0.1996\\ 
    0.802\\ 
    0.7124
    \end{bmatrix} \right)
\end{aligned}
\label{eqn:cluster:res}
\end{equation}

\section{Generative Polyra Swarm}\label{app:generative}
The Hit-and-Run algorithm~\cite{hitandrun} is a method to generate samples uniformly from a convex shape. Starting with a point in the shape, it picks a direction at random. A line through this point along the chosen direction intersects with the border of the shape exactly twice (as the shape is convex), and we pick a random point between these intersections. The process is iterated starting from this new point until the desired amount of samples is generated.

To generalize this behavior to arbitrary shapes, means to consider the case where the number of intersections is larger than two. We have already created an algorithm to find the intersection points along a direction in Appendix~\ref{app:abs:1d}. Additionally, we weigh the likelihood of a point in each linear segment by the inverse of its length to achieve somewhat uniformly distributed samples.

\subsection{Generating Images}\label{app:zerogen}
Similar to the classification case, we only show the results of our generative model for a toy dataset in the main paper. Here, we want to consider the more realistic case by considering MNIST~\cite{mnist} data; namely, we want to generate images of the digit zero. 

When training on raw pixel data with $28²=784$ features, we see that every generated image is exactly the same (or very close) as the initial image chosen. This shows a major fault of our generative algorithm. When the learned shape contains many disconnected parts, we need to randomly select a direction that intersects with one of these disconnected islands. When the volume of the learned shape is small and the dimension of the data is high, this is extremely unlikely and thus we never leave the initial island, resulting in very similar generated samples.

To circumvent this, we again rely on a PCA algorithm to reduce the dimensionality of our feature space. Using this (and the hyperparameters shown in Table~\ref{tab:h:zerosgen}), we can generate somewhat realistically looking zeros in Figure~\ref{fig:pcagen}. While these are not optimal, with some versions looking more like intersecting rings instead of the digit zero, they are also all new, unique samples. Further it is likely possible to improve this further by using a better feature representation algorithm (like an autoencoder~\cite{aeusedimensionality}).

\begin{figure}[H]
    \centering
    \includegraphics[width=0.6\linewidth]{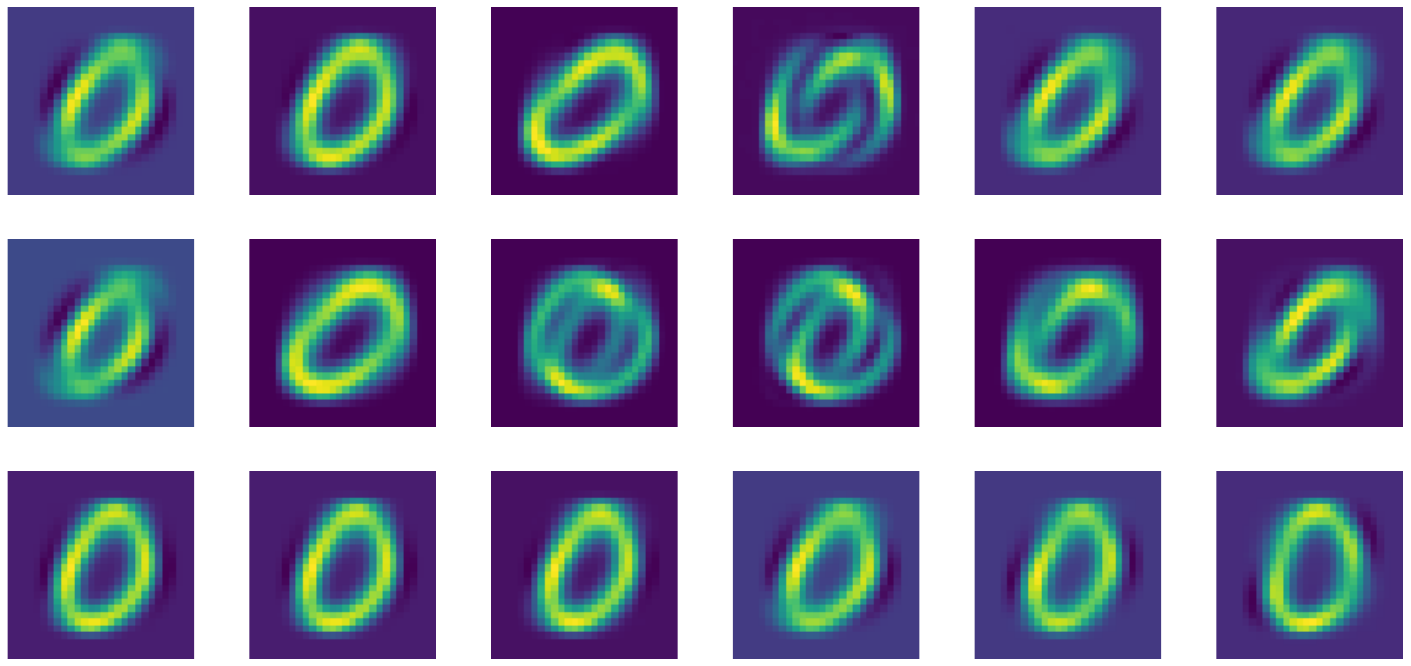}
    \caption{Images of MNIST zeros generated by a Polyra Swarm. Because both Polyra Swarms and the generative algorithm don't work well on high-dimensional data, we use here a PCA algorithm to reduce the dimensionality of the data. The results would likely look even better when using a deep learning-based feature extractor like an autoencoder~\cite{aeusedimensionality}. The hyperparameters of this experiment can be found in Table~\ref{tab:h:zerosgen}.}
    \label{fig:pcagen}
\end{figure}

\begin{table}[h]
\centering
\begin{tabular}{|lr|}
\hline
\textbf{Hyperparameter} & \textbf{Value} \\
\hline
Sample count     & $5923$  \\
minpoi           & $1000$  \\
PCA components           & $3$  \\

\hline
\end{tabular}
\caption{Non-default hyperparameter settings used in the image generation experiment.}
\label{tab:h:zerosgen}
\end{table}

\section{Polyra-based Optimization}\label{app:opt}
While Polyra Swarms are designed to not require optimization, this does not mean that we can't use them to find solutions to optimization problems. However, like with most of the applications in Section~\ref{sec:use}, it requires rethinking how we approach these tasks. We show this with Kepler's equation (Equation~\ref{eqn:kepler}). While this equation is technically simple, finding an inverse solution $E(M)$ is still part of active research~\cite{keplersequation} and is most easily done by numerical methods minimizing Equation~\ref{eqn:kepler:loss}.
\begin{equation}
    M(E)=E-e\sin(E)\hspace{3em} e=1.0
    \label{eqn:kepler}
\end{equation}
\begin{equation}
    L=\|M(E)-(E-e\sin(E))\|
    \label{eqn:kepler:loss}
\end{equation}

We can do something similar using a Polyra Swarm, by searching for the shape created by Equation~\ref{eqn:kepler:loss} between $L$,$M$ and $E$. By afterward demanding that $L=0$, we can find a shape representing the relationship between $M$ and $E$. This resulting shape (assuming the hyperparameters given in Table~\ref{eqn:kepler:loss}) is shown in Figure~\ref{fig:kepler}. While this is by no means the fastest (or most accurate) way of inverting Equation~\ref{eqn:kepler}, it still shows the benefit of having ensembles of small submodels that are easy to manipulate. And instead of finding a singular solution we effectively find almost all solutions at the same time.

\begin{figure}[H]
    \centering
    \includegraphics[width=0.6\linewidth]{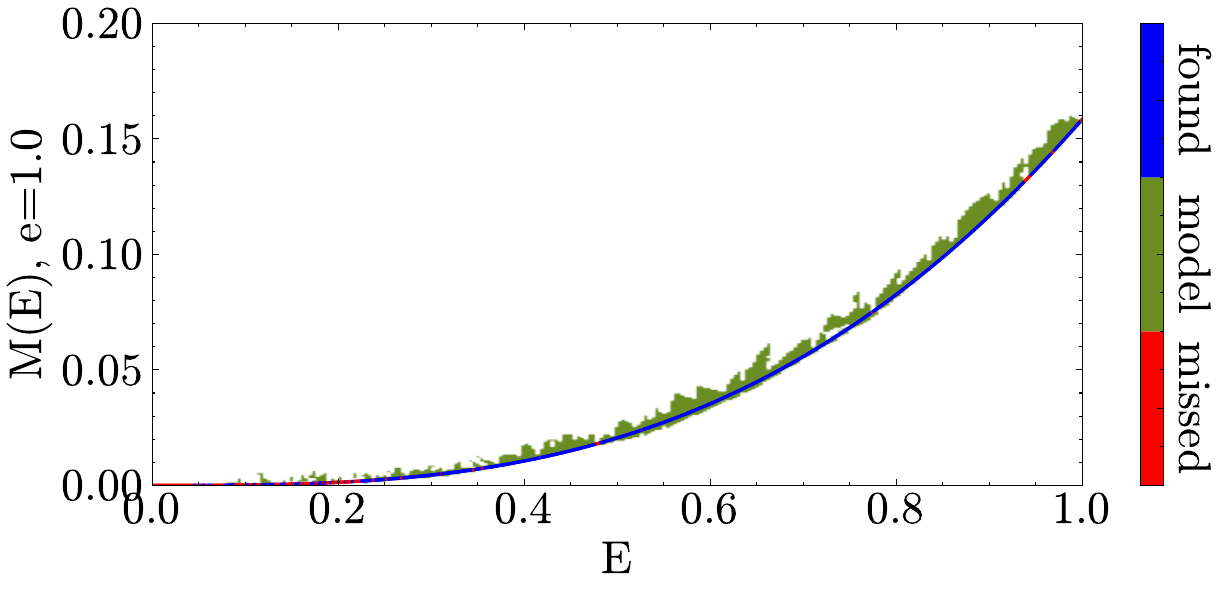}
    \caption{Example of how to use a Polyra Swarm to find solutions to Equation~\ref{eqn:kepler}. We show the region that the swarm believes to be part of the solution in green, and the actual solution in either blue or red, depending on whether the swarm believes this equation to be valid.}
    \label{fig:kepler}
\end{figure}

\begin{table}[h]
\centering
\begin{tabular}{|lr|}
\hline
\textbf{Hyperparameter} & \textbf{Value} \\
\hline
Sample count     & $100000$  \\
minpoi           & $1000$  \\
extend           & $0.015$  \\

\hline
\end{tabular}
\caption{Non-default hyperparameter settings used in the optimization experiment.}
\label{tab:h:kepler}
\end{table}

\section{Setup - Mandelbrot fit}\label{app:mandelbrot}

To achieve the highly accurate fit shown in Figure~\ref{fig:eye}, the largest challenges stem from the very large dataset we use. We generate a dataset from points that lie in the mandelbrot set\footnote{Technically, the mandelbrot set has a boundary of infinite length and thus does not fulfill our Definition~\ref{def:shape} to qualify as a shape. Thus we more accurately try to approximate the shape of the set describing a given finite resolution of samples following the mandelbrot set.} (don't violate $\|Z\|\le2$ in $100$ steps). We generate $50000$ values in each direction (Thus $50,000^2=2,500,000,000$ samples in total) from $-2\le x\le1.5$ and $-2\le y\le2$. All points are used as a test set, while only the points that seem to lie in the mandelbrot set are used for training the Polyra Swarm(roughly $10\%$). This setup makes false negatives impossible since no training sample can lie considered outside of the shape(Appendix~\ref{app:alltrueproof}). We chose this here to show the low bias of our model and also because repeating this experiment would have been computationally infeasible. In fact, to train this model, we heavily parallelize our computation, by noticing that we can combine the output of two Polyra Swarms as $p=\min(p_1,p_2)$. To focus on the small details of the mandelbrot set, we choose a higher $Adim=6$. The full hyperparameters can be found in Table~\ref{tab:h:mandelbrot}. We train $40000$ chunks of $25$ submodels each. Each of our submodels takes about $40min$ because of the extreme number of training samples (over $250$ million). In total this results in a training time of approximately three years of CPU time.

 \begin{table}[h]
\centering
\begin{tabular}{|lr|}
\hline
\textbf{Hyperparameter} & \textbf{Value} \\
\hline
Sample count     & $276147856$  \\
Model count       & $1000000$   \\
Minpoi      & $10$   \\
Adim      & $6$   \\
\hline
\end{tabular}
\caption{Non-default hyperparameter settings used in the mandelbrot experiment.}
\label{tab:h:mandelbrot}
\end{table}

\section{Setup - Matching Algorithm}\label{app:match}
As an example of how to use shape approximation for a task that is classically solved by function approximation, we consider matching tasks (e.g. a dating app algorithm, matching advertisement to the users who are the most likely to engage with it, ...). As an example we consider a situation where both participants to be matched are described by singular features. This means a successful match is described by a point in two-dimensional space. Thus, we search for the shape describing this region to see what potential successful matches could look like. Our hyperparameters for this can be found in Table~\ref{tab:h:match}.

Furthermore, we can input known information of one participant to reduce the dimensionality of the search space. Using the abstraction procedure described in Appendix~\ref{app:abs}, we can simplify its description to create a personal description of how a successful match would look like for a given participant.

\begin{table}[h]
\centering
\begin{tabular}{|lr|}
\hline
\textbf{Hyperparameter} & \textbf{Value} \\
\hline
Sample count     & $10000$  \\
Model count       & $5000$   \\
Minpoi      & $2500$   \\
Extend      & $0.035$   \\
\hline
\end{tabular}
\caption{Non-default hyperparameter settings used in the matching algorithm experiment.}
\label{tab:h:match}
\end{table}

\section{Setup - Classification}\label{app:classification}
We compare our classification setup qualitatively on a common toy dataset. For this, we use Polyra Swarms (which hyperparameters can be found in Table~\ref{tab:h:class}) and a neural network (Whose structure can be found in Table~\ref{tab:h:class:nn:setup} and whose hyperparameters can be found in Table~\ref{tab:h:class:nn}). We further include a reject option for the neural network, where uncertain values with probabilities between $0.05-0.95$ are rejected.

To use Polyra Swarms for classification, we use one Polyra Swarm for each class. This produces a function that measures association for each class. If there is exactly one class that a sample seems to associate with, we classify the sample as this class. However it is possible that more or fewer classes consider a sample to match this group. This situation and an example of classification for real-world datasets are discussed in Appendix~\ref{app:mnistclass}.

\begin{table}[h]
\centering
\begin{tabular}{|lr|}
\hline
\textbf{Hyperparameter} & \textbf{Value} \\
\hline
Sample count     & $2500\times2$  \\
Model count       & $1000\times2$   \\
Minpoi      & $500$   \\
Extend      & $0.1$   \\
\hline
\end{tabular}
\caption{Non-default hyperparameter settings used for polyra in the classification experiment.}
\label{tab:h:class}
\end{table}

\begin{table}[h]
\centering
\begin{tabular}{|lr|}
\hline
\textbf{Hyperparameter} & \textbf{Value} \\
\hline
Epochs     & $100$  \\
Batchsize       & $32$   \\
Patience      & $5$   \\
Learning rate      & $0.001$   \\
Optimizer      & Adam   \\
\hline
\end{tabular}
\caption{Hyperparameter settings used for the neural network in the classification experiment.}
\label{tab:h:class:nn}
\end{table}
\begin{table}[h]
\centering
\begin{tabular}{|lll|}
\hline
\textbf{Layer} & \textbf{Type}           & \textbf{Details} \\
\hline
1   & Input            & shape = 2 \\
2-4 & Dense (3x)       & 64 units, ReLU activation \\
5   & Dense (Output)   & 2 units, Softmax activation \\
\hline
\end{tabular}
\caption{Neural network architecture used in the classification experiment.}
\label{tab:h:class:nn:setup}
\end{table}

\section{Setup - Anomaly Detection}\label{app:anomaly}

For our experiment on using Polyra Swarms for anomaly detection, we follow a recent anomaly detection benchmarking paper~\cite{surveyzhao}. Following this, we use $121$ common anomaly detection benchmarking datasets, extracted from various situations. Thus, while every dataset is converted into a tabular format, some contain medical data and some image or text-based data. We choose as our competitors recent deep learning anomaly detection algorithms. These include NeuTral~\cite{NeuTralAD}, Diffusion Time Estimation (DTE)~\cite{dte}, GOAD~\cite{goad}, DAGMM~\cite{dagmm}, DeepSVDD~\cite{deepsvdd} and Normalizing Flow~\cite{nf}. We use the implementation and hyperparameters of either pyod~\cite{pyod} or the original papers. We further consider each algorithm in the one-class setting, where only access to normal samples is given.

Anomaly detection is usually evaluated using continuous anomaly scores and not using the binary decisions (normal or not normal) Polyra Swarms generate. Because continuous scores generally result in a higher anomaly score (See Appendix~\ref{app:rocbias}), we need to propose a way of generating anomaly scores from a Polyra Swarm. These are less than optimal since, e.g., abstraction generally changes these scores. However, it is also rare that anomaly detection requires continuous scores in practice. Generally, anomaly detection algorithms also propose continuous scores to allow for varying thresholds, but it is also possible to control these thresholds through hyperparameters.

We propose here two ways of generating continuous anomaly scores:

\begin{equation}
    \textrm{score}_1(x)=\textrm{mean}_{i} (f_i(x))
    \label{eqn:anoscore1}
\end{equation}
\begin{equation}
    \textrm{score}_2(x)=\frac{\textrm{mean}_{i} (x\in (A_i\cap B_i))}{\textrm{mean}_{i} (x\in A_i)}
    \label{eqn:anoscore2}
\end{equation}

Additionally, since we now only calculate fractions instead of the most extreme decision by every model, the dependency of the volume on the number of submodels (as shown in Appendix~\ref{app:hyper:ensemblesize}) does no longer apply. Instead using larger ensembles only reduces random noise and thus increases average performance. This is shown in Table~\ref{tab:adperf}. In fact, when using even more submodels than the $400000$ used here, the anomaly detection performance would likely still improve. However, since these these experiments take around $1h$ of average computation time per dataset (for a total time of roughly five days), we did not test this further. 

It seems like the second anomaly score (Equation~\ref{eqn:anoscore2}), which calculates the number of submodels where the submodel is fulfilled divided by the number of submodels where the sample is included in the condition polytope, performs better. We use the largest ensemble with the second anomaly score in Figure~\ref{fig:better}. However, the difference to the simpler to calculate first anomaly score (Equation~\ref{eqn:anoscore1}) is small for large ensemble sizes. Also, it is likely that by varying hyperparameters, we could increase the performance further (Table~\ref{tab:h:anomaly}).

\begin{table}[h]
\centering
\begin{tabular}{|r|r|r|}
\hline
\textbf{Size} & \textbf{Performance (1)} & \textbf{Performance (2)} \\
\hline
1000   &  $0.7253$  &  $0.7395$  \\
10000  &  $0.7641$  &  $0.7683$  \\
100000 &  $0.7814$  &  $0.7839$  \\
400000 &  $0.7842$  &  $\textbf{0.7887}$  \\
\hline
\end{tabular}
\caption{Average anomaly detection performance with different ensemble sizes and anomaly score functions.}
\label{tab:adperf}
\end{table}

\begin{table}[h]
\centering
\begin{tabular}{|lr|}
\hline
\textbf{Hyperparameter} & \textbf{Value} \\
\hline
Sample count     & variable  \\
Model count      & up to $400000$   \\
\hline
\end{tabular}
\caption{Non-default hyperparameter settings used in the anomaly detection experiment.}
\label{tab:h:anomaly}
\end{table}

\section{Setup - Regression}\label{app:regression}
We approximate the shape spanned by the function $\sin(x)$ and an uncertainty of $5\%$ using the hyperparameters shown in Table~\ref{tab:h:regression}. Using this shape, we find the region inside the shape for a given input value (here $x=\pi$). The ground truth is given by $\sin(\pi)=0$. The learned region ranging from $-0.070$ and $0.054$ also represents the $5\%$ uncertainty well. We find these (exact) limits with the one-dimensional abstraction described in Appendix~\ref{app:abs:1d}.

\begin{table}[h]
\centering
\begin{tabular}{|lr|}
\hline
\textbf{Hyperparameter} & \textbf{Value} \\
\hline
Sample count     & $10000$  \\
Minpoi      & $1000$   \\
Extend      & $0.0125$   \\
\hline
\end{tabular}
\caption{Non-default hyperparameter settings used in the regression experiment.}
\label{tab:h:regression}
\end{table}

\section{Setup - Turkish Flag}\label{app:turk}
We describe the algorithm generating random points in Appendix~\ref{app:generative}. Here, we test the distribution of points by first learning a shape representing the white part of the Turkish flag (Using the hyperparameters in Table~\ref{tab:h:turk}) and then generating points in this shape. As Figure~\ref{fig:turk} shows, these points match quite well to the original flag.
\begin{table}[h]
\centering
\begin{tabular}{|lr|}
\hline
\textbf{Hyperparameter} & \textbf{Value} \\
\hline
Sample count     & $27052$  \\
Model count      & $5000$   \\
extend           & $0.025$  \\
minpoi           & $1000$  \\
Generated Samples           & $500$  \\

\hline
\end{tabular}
\caption{Non-default hyperparameter settings used in the Turkish flag experiment.}
\label{tab:h:turk}
\end{table}


\section{Setup - Diamond Abstraction}\label{app:diamond}

To show the effect of abstraction, we generate samples uniformly distributed in a $45°$ rotated rectangle. We fit a Polyra Swarm with the hyperparameters in Table~\ref{tab:h:diamond} to this data and abstract it with sample abstraction (Appendix~\ref{app:abs:sample}) using default parameters.

\begin{table}[h]
\centering
\begin{tabular}{|lr|}
\hline
\textbf{Hyperparameter} & \textbf{Value} \\
\hline
Sample count     & $10000$  \\
Model count      & $2000$   \\

\hline
\end{tabular}
\caption{Non-default hyperparameter settings used in the diamond abstraction experiment.}
\label{tab:h:diamond}
\end{table}

\section{Setup - Old Faithful Abstraction}\label{app:faith}

To further show the benefits of abstraction, we use the OldFaithful dataset~\cite{oldfaithful}. This dataset contains a link between the time between eruptions and the length of the following eruption for the Old Faithful Geyser in Wyoming. We fit a Polyra Swarm with the hyperparameters in Table~\ref{tab:h:faith} to this data and abstract it with sample abstraction (Appendix~\ref{app:abs:sample}) using default parameters.

\begin{table}[h]
\centering
\begin{tabular}{|lr|}
\hline
\textbf{Hyperparameter} & \textbf{Value} \\
\hline
Sample count     & $272$  \\
Model count      & $100$   \\
extend           & $0.1$  \\
minpoi           & $100$  \\

\hline
\end{tabular}
\caption{Non-default hyperparameter settings used in the old faithful experiment.}
\label{tab:h:faith}
\end{table}

\section{Setup - Text Approximation}\label{app:textapprox}

Finally, to show the precision of a Polyra Swarm, we generate data following the word "Polyra". We want to decide if a two-dimensional point is inside the region covered by this text (There would be ink at this point when writing the word "Polyra" on Paper). We do this first using only samples inside the text and a Polyra Swarm with the hyperparameters in Table~\ref{tab:h:text}. Additionally, we do the same with both samples inside and outside of the text using a neural network with the setup described in Table~\ref{tab:h:text:nn:setup} and the hyperparameters described in Table~\ref{tab:h:text:nn}. For the neural network fit, we try to simplify the task as much as possible: First, we use both negative and positive samples to not have to rely on one-class classification algorithms, which tend to perform worse than algorithms with more explicit supervision~\cite{survey-ruff}. Additionally, we use the same data for both training and testing in both cases, which makes overfitting impossible to detect. However still, we observe that the neural network regression understands the structure of the text significantly worse compared to the Polyra Swarm. We interpret this as the neural network optimization reaching a local minimum and a benefit of the fact that Polyra Swarms do not require optimization and thus can not get stuck in local minima.

\begin{table}[h]
\centering
\begin{tabular}{|lr|}
\hline
\textbf{Hyperparameter} & \textbf{Value} \\
\hline
Sample count     & $59240$  \\
Model count      & $10000$   \\
extend           & $0.01$  \\
minpoi           & $500$  \\

\hline
\end{tabular}
\caption{Non-default hyperparameter settings used for polyra in the text approximation experiment.}
\label{tab:h:text}
\end{table}

\begin{table}[h]
\centering
\begin{tabular}{|lr|}
\hline
\textbf{Hyperparameter} & \textbf{Value} \\
\hline
Sample count & $280053$ \\
Epochs     & $100$  \\
Batchsize       & $32$   \\
Patience      & $5$   \\
Learning rate      & $0.001$   \\
Optimizer      & Adam   \\
\hline
\end{tabular}
\caption{Hyperparameter settings used for the neural network in the text approximation experiment.}
\label{tab:h:text:nn}
\end{table}
\begin{table}[H]
\centering
\begin{tabular}{|lll|}
\hline
\textbf{Layer} & \textbf{Type}           & \textbf{Details} \\
\hline
1   & Input            & shape = 2 \\
2-4 & Dense (3x)       & 64 units, ReLU activation \\
5   & Dense (Output)   & 2 units, Softmax activation \\
\hline
\end{tabular}
\caption{Neural network architecture used in the text approximation experiment.}
\label{tab:h:text:nn:setup}
\end{table}

\section{Hardware setup}\label{app:hardware}
For every experiment, except the mandelbrot fit (Figure~\ref{fig:eye}) and the anomaly detection performance (Figure~\ref{fig:better}), we use only a consumer-grade laptop to conduct our experiments. We use a Thinkpad T14 with an Intel i7-1260p (16 cores) and $32gb$ RAM and every experiment takes less than $10\textrm{min}$. No GPU is used during any experiments.

The anomaly detection experiment (Figure~\ref{fig:better}) is conducted on a server with an Intel Xeon w9-3495X CPU with 112 cores and $512gb$ of RAM. The mandelbrot fit is conducted on a cluster with varying CPU's and $40gb$ of RAM per process. These experiments also do not use any GPU and more details about their runtime can be found in Appendix~\ref{app:anomaly} and Appendix~\ref{app:mandelbrot}.


\end{document}